\documentclass[journal,onecolumn]{IEEEtran}
\usepackage[english]{babel}
\usepackage{algorithm,algorithmicx}
\usepackage{tcolorbox}
\usepackage{tikz}

\usepackage{framed}
\usepackage{algpseudocode}
\usepackage{amsthm,nccmath}
\usepackage{amsmath,bm}
\usepackage{bigints}
\usepackage{amssymb}
\usepackage{enumitem}
\usepackage{graphicx}
\usepackage{graphics}
\usepackage{tabularx}
\usepackage{multirow}
\usepackage[font=small,labelfont=bf]{caption}
\usepackage[font=small]{subcaption}
\usepackage{float}
\usepackage{epstopdf}
\usepackage{footnote}
\makesavenoteenv{tabular}
\makesavenoteenv{table}
\usepackage{bbm}
\usepackage{amsthm}
\usepackage{stmaryrd}
\usepackage{amsmath,bm}
\usepackage{amssymb}
\usepackage{mathtools, cuted}
\usepackage{kantlipsum,setspace}
\usepackage[normalem]{ulem}

\usepackage{dblfloatfix}
\theoremstyle{plain}
\newtheorem{theorem}{Theorem}[section]
\newtheorem{lemma}[theorem]{Lemma}
\newtheorem*{lemma*}{Lemma}
\newtheorem{prop}[theorem]{Proposition}

\newtheorem*{cor*}{Corollary}

\theoremstyle{definition}

\newtheorem*{defn*}{Definition}

\theoremstyle{remark}

\newtheorem{assump}{Assumption}
\newtheorem*{assump*}{Assumption}

\DeclareMathOperator*{\argmin}{arg\,min}

\newcommand{\mbb}{\mathbb}
\newcommand{\mbf}{\mathbf}

\newcommand{\br}{\mathcal B^{\mathrm{r}}}
\newcommand{\brt}{\mathcal B_t^{\mathrm{r}}}
\newcommand{\bc}{\mathcal B^{\mathrm{c}}}
\newcommand{\bct}{\mathcal B_t^{\mathrm{c}}}
\newcommand{\muR}{\mu_{\mathrm{r}}}
\newcommand{\muC}{\mu_{\mathrm{c}}}
\newcommand{\muCi}{\mu_{\mathrm{c},i}}
\newcommand{\nR}{n_{\mathrm{r}}}
\newcommand{\nC}{n_{\mathrm{c}}}
\newcommand{\dnr}{d_{\mathrm{nr}}}

\newcommand{\lp}{\left(}
\newcommand{\rp}{\right)}

\newcommand{\lnr}{\left\|}
\newcommand{\rnr}{\right\|}

\DeclareMathOperator*{\minimize}{\text{minimize}}

\usepackage{xcolor}

\newcommand{\mycomment}[1]{}

\usepackage{adjustbox}
\usepackage{wrapfig}
\usepackage{booktabs}
\usepackage{multirow,mathtools} 
\usepackage{threeparttable}

\usepackage[switch]{lineno}  %

\usepackage{tikz}
\usetikzlibrary{calc,trees,positioning,arrows,chains,shapes.geometric,%
	decorations.pathreplacing,decorations.pathmorphing,shapes,%
	matrix,shapes.symbols}

\tikzstyle{startstop} = [rectangle, draw, rounded corners, align=center, minimum width=3cm, minimum height=1cm,text centered]
\tikzstyle{decision} = [diamond, draw, fill=blue!20, 
text width=4.5em, text badly centered, node distance=3cm, inner sep=0pt]
\tikzstyle{block} = [rectangle, draw, fill=blue!10, align=center, rounded corners, minimum width=3cm, minimum height=1cm]
\tikzstyle{blockcast} = [rectangle, draw, fill=red!25, align=center, rounded corners, minimum width=3cm, minimum height=0.45cm]
\tikzstyle{line} = [draw, -latex']
\tikzstyle{cloud} = [draw, ellipse,fill=red!20, node distance=3cm,
minimum height=2em]

\title{Zeroth-Order Hybrid Gradient Descent: Towards A Principled Black-Box Optimization Framework}

\author{
	\IEEEauthorblockN{Pranay Sharma, Kaidi Xu, Sijia Liu, \\ 
	Pin-Yu Chen, Xue Lin and Pramod K. Varshney.}
	\thanks{P. Sharma and P. K. Varshney are with the Department of Electrical Engineering and Computer Science, Syracuse University, Syracuse, NY-13244 (email: {psharm04, varshney}@syr.edu). K. Xu and X. Lin are with the Department of Computer and Electrical Engineering, Northeastern University, Boston, MA-02115 (email: {xu.kaid, xue.lin}@northeastern.edu). S. Liu and P. Chen are with MIT-IBM Watson AI Lab, IBM Research, Cambridge, MA-02142 (e-mail: {sijia.liu, pin-yu.chen}@ibm.com).}
}

\begin{document}
\maketitle
\begin{abstract}
    In this work, we focus on the study of stochastic zeroth-order (ZO) optimization which does not require first-order gradient information and uses only function evaluations. The problem of ZO optimization has emerged in many recent machine learning applications, where the gradient of the objective function is either unavailable or difficult to compute. In such cases, we can approximate the full gradients or stochastic gradients through function value based gradient estimates. Here, we propose a novel hybrid gradient estimator (HGE), which takes advantage of the query-efficiency of random gradient estimates as well as the variance-reduction of coordinate-wise gradient estimates. We show that with a graceful design in coordinate importance sampling, the proposed HGE-based ZO optimization method is efficient both in terms of iteration complexity as well as function query cost. We provide a thorough theoretical analysis of the convergence of our proposed method for non-convex, convex, and strongly-convex optimization. We show that the convergence rate that we derive generalizes the results for some prominent existing methods in the nonconvex case, and matches the optimal result in the convex case. We also corroborate the theory with a real-world black-box attack generation application to demonstrate the empirical advantage of our method over state-of-the-art ZO optimization approaches.
\end{abstract}


\IEEEpeerreviewmaketitle

\section{Introduction}
\IEEEPARstart{D}{erivative}-free optimization (DFO) methods \cite{conn09introduction, larson19dfo} have become increasingly popular in recent years, owing to the advent of several machine learning applications where the analytical expressions of the objective functions are either expensive or infeasible to obtain. Some examples of such applications are black-box adversarial example generation in deep neural networks (DNNs), reinforcement learning, and control and management of time-varying networks with limited computational resources \cite{liu2020primer}.

Zeroth-order (ZO) methods  form a special class of DFO methods which can be seen as gradient-less versions of first-order (gradient-based) optimization methods. 
ZO optimization involves approximating the full/stochastic gradient of the function using only the function values, and using this gradient estimator in the first-order (FO) optimization framework. 
Advantages of ZO-methods, over conventional DFO methods like direct-search based methods \cite{bortz98simplex}, and trust-region methods \cite{conn2000trust}, are  the ease of implementation and the convergence properties of these methods, owing to their theoretical \textit{closeness} to FO methods. 
ZO methods in the literature often have convergence rates comparable to FO methods, with an additional small-degree polynomial in the problem dimension $d$ \cite{duchi15optimal_tit, nesterov17grad_free}.

However, ZO algorithms often suffer from the high variance of   gradient estimates. 
The existing estimators involve choosing between  saving on the function query cost \cite{Ghadimi_Siam_2013_SGD}, and achieving higher accuracy of the gradient estimates \cite{lian2016comprehensive}. 
In this work, we propose a novel gradient 
21
 estimator which traverses the entire spectrum between these two extremes. 
Essentially, the proposed algorithm, based on our novel estimator, improves the variance of the gradient estimator, while also saving on the function query budget.

The general ZO approach  involves computing a gradient estimate $\Hat{\nabla} f(\mbf x)$ of     function $f$ at the point $\mbf x$, and then plugging this estimate into a FO method. One way to estimate the gradient is by querying the function at a single randomly chosen point in the vicinity of $\mbf x$ \cite{flaxman05acm_siam}.
More effectively, multi-point (e.g., {two-point}) approaches are used \cite{agarwal10colt, nesterov17grad_free}, leading to better variance and improved complexity results. 

The initial work on multi-point estimators was largely limited to convex problems. 
For smooth, deterministic problems, \cite{nesterov17grad_free} proposed the ZO gradient descent (ZO-GD) algorithm and proved $O(d/T)$ convergence rate, where $d$ denotes the problem size and $T$ is the number of iterations. 
A ZO-mirror descent algorithm \cite{duchi15optimal_tit} extended this to the stochastic case, achieving the rate of $O(\sqrt{d}/\sqrt{T})$. 
The authors also proved the result to be \textit{order-optimal}. For nonsmooth problems, \cite{shamir17optimal_jmlr} proved the same rate to be optimal, while also extending the analysis to non-Euclidean problems.

In the nonconvex domain, the first stochastic algorithm, ZO-SGD \cite{Ghadimi_Siam_2013_SGD} utilized vectors sampled from normal distribution, and achieved $O(\sqrt{d}/\sqrt{T})$ convergence rate. 
The same rate was achieved by \cite{gao18information}, while using random vectors sampled from the surface of the unit sphere. 
In \cite{lian2016comprehensive}, an asynchronous ZO-SCD approach was proposed for parallel architecture settings, which achieved $O(\sqrt{d}/\sqrt{T})$ rate.
Following the recent progress in variance reduction methods for first-order optimization: SAGA \cite{defazio14saga}, SVRG \cite{Johnson_NIPS_2013, Reddi_ICML_2016}, SARAH \cite{Nguyen_ICML_2017_SARAH}, SPIDER \cite{Fang_NIPS_2018_spider}, to name a few, the ZO extensions of variance reduced methods have also appeared in recent years. 
ZO-SVRG improved the iteration complexity to $O(d/T)$, but at the expense of an increased function query complexity. The iteration complexity is further improved in \cite{Fang_NIPS_2018_spider,zo_spider_liang19icml}.

ZO-counterparts of FO methods have also been proposed in other contexts, such as constrained optimization \cite{lan18zeroth_conditional}, adaptive momentum methods \cite{chen2019zo_adamm}, mitigation of extreme components of gradient noise \cite{liu2019signsgd}, and distributed optimization over networks \cite{hajinezhad19zone, tang19allerton}.
ZO optimization has recently been shown to be   powerful   in evaluating the adversarial robustness of deep neural networks (DNNs), by generating black-box adversarial examples,  e.g., crafted images with imperceptible perturbations, to deceive a well-trained
DNN using only input-output model queries \cite{chen2017zoo,ilyas2018black,ilyas2018prior,cheng2018query,liu2020min,tu2018autozoom}. 
The internal configurations of the victim DNN systems are not revealed to the attackers and the only mode of interaction with the systems is by submitting inputs and receiving the predicted outputs.

\subsection{Our Contributions} \label{subsec_contri}
We summarize our contributions below:
\begin{itemize}
    \item 
    We propose a novel function value based gradient estimator (we call HGE, \underline{h}ybrid \underline{g}radient \underline{e}stimator), which  takes advantage of both the query-efficient random  gradient estimate and the variance-reduced coordinate-wise  gradient estimate. We also develop a coordinate importance sampling method to further improve the variance of HGE under a fixed number of function queries.
    
    \mycomment{We have proposed a novel gradient estimator for zeroth-order optimization. Our estimator generalizes ZO-SGD and ZO-SCD, and has reduced variance compared to both, without relying on the conventional variance reduction methods, which have higher function query complexity.}
    \item 
    We propose a ZO hybrid gradient descent (ZO-HGD) optimization method with the aid of HGE. We show that ZO-HGD is general since it covers ZO stochastic gradient descent (ZO-SGD) \cite{Ghadimi_Siam_2013_SGD}  and ZO stochastic coordinate descent (ZO-SCD) \cite{lian2016comprehensive} as special cases.  
    We provide a comprehensive theoretical analysis for the convergence of ZO-HGD across different optimization domains, showing that   ZO-HGD is efficient in both iteration and function query complexities.
    
    \mycomment{Our estimator combines the advantages of both ZO-SGD and ZO-SCD, at no additional function query cost. The coordinate selection is based on importance sampling (rather than uniform sampling as is the case in ZO-SCD). The probability computation for importance sampling is inspired by the gradient quantization literature in distributed  first-order optimization.}
    \mycomment{
    \item 
    We have done comprehensive theoretical analysis of our method for smooth nonconvex, convex and strongly convex functions. We also demonstrate existing methods as special cases of our approach.}
    \item 
    We demonstrate the effectiveness of ZO-HGD through a real-world application to generating  adversarial examples from a black-box deep neural network \cite{chen2017zoo,ilyas2018black}. We show that ZO-HGD outperforms ZO-SGD, ZO-SCD and ZO sign-based SGD methods in striking a graceful balance between query efficiency and attack success rate. 
    \mycomment{We supplement our theoretical results with extensive simulations on a synthetic dataset as well as a real-world application to adversarial example generation.}
\end{itemize}

The paper is organized as follows. In Section \ref{sec_prelim}, we state the problem, and discuss two of the existing zeroth-order gradient estimators. We propose our hybrid estimator in Section \ref{sec_HGE}. This estimator is used to propose a zeroth-order hybrid gradient descent (ZO-HGD) algorithm in Section \ref{sec_HGD}. We discuss the convergence properties of the algorithm for smoooth, nonconvex functions in Section \ref{sec_HGD}, and for smooth, convex and strongly convex functions in Section \ref{sec_convex}. We provide the experimental results in Section \ref{sec_experiments}, followed by conclusion in Section \ref{sec_conclusion}. All the proofs are deferred to the Appendix.

\section{Preliminaries}
\label{sec_prelim}
In this section, we begin by presenting the formulation of the black-box optimization problems of our interest. We then review two commonly-used ZO gradient estimators, random gradient estimator ({RGE}) and coordinate-wise gradient  estimator ({CGE}). 
We shall define the notation wherever we introduce a new mathematical entity. We refer the reader to Appendix \ref{app_notation} for a tabular summary of all the notations used.

We consider the following  black-box  stochastic optimization problem
\begin{align}
\label{eq_problem}
    \min_{\mathbf{x} \in \mathbb{R}^d} f(\mathbf{x}) \triangleq \mathbb{E}_{\xi \sim \Xi} F(\mathbf{x}; \xi),
\end{align}
where $\mathbf x$ denotes the $d$-dimensional optimization variable,
$\xi \in \Xi$ denotes a stochastic variable with  distribution $\Xi$ (e.g., distribution of training samples), 
and $F(\cdot; \xi)$ is a smooth, possibly nonconvex loss function. 
By black-box, we mean that the objective function in \eqref{eq_problem} is only accessible via functional evaluations. 
To enable theoretical analysis, we impose two commonly-used assumptions on problem \eqref{eq_problem}.
\begin{assump}{\textit{Gradient Lipschitz continuity:}}
\label{assum_lipschitz}
The  loss function $f$ in \eqref{eq_problem} has Lipschitz continuous gradient with parameter $L$, i.e., $f(\mathbf{y}) \leq f(\mathbf{x}) + \langle \nabla f(\mathbf{x}), \mathbf{y} - \mathbf{x} \rangle + \frac{L}{2} \lnr \mathbf{y} - \mathbf{x}  \rnr ^2_2$, for all $\mbf x, \mbf y$, where $\nabla f$ denotes the gradient of $f$.
%
\end{assump}
\begin{assump}{\textit{Bounded variance of stochastic gradients:}}
\label{assum_var_bound}
 Suppose $(\mathbf x)_i$ denotes the $i$th coordinate of a vector $\mathbf x$, and $\zeta$ is a given constant, then $\mathbb{E} [ ( \nabla F(\mathbf{x}; \xi) - \nabla f(\mathbf{x}) )_i^2 ] \leq \zeta^2, \forall \ i$.
%
%
\end{assump}
Assumption \ref{assum_lipschitz} is fairly standard in the theoretical analysis of nonconvex optimization \cite{Ghadimi_Siam_2013_SGD}. 
Assumption \ref{assum_var_bound} enables a finer control on the variance of CGE \cite{lian2016comprehensive,signsgd18icml}. It also implies that $\mathbb{E} \left\| \nabla F(\mathbf{x}; \xi) - \nabla f(\mathbf{x}) \right\|^2 \leq \sigma^2 \triangleq d \zeta^2$.

\paragraph{{RGE.}}
Considering a set of  \textit{random} directional vectors $\{ \mathbf u_i \}_{i=1}^{\nR}$, RGE of the individual loss function $F(\mathbf x; \xi)$ is given by the average of the finite difference approximations of the directional derivatives of $F(\mathbf x; \xi)$ along these random directions \cite{gao18information,liu2018zeroth}:
\begin{align} 
    \Hat{\nabla}_{\mathrm{RGE}} F (\mathbf x; \xi)  =   \frac{1}{\nR} \sum_{i=1}^{\nR} \frac{d \left [ F(\mathbf{x} + \muR \mathbf{u}_{i}; \xi) - F(\mathbf{x}; \xi) \right ]}{\muR}  \mathbf{u}_{i},  
    \label{eq_RGE}
\end{align}
where $\muR > 0$  is a perturbation radius (also called smoothing parameter), and  each $\mathbf u_i$ is drawn from the uniform distribution on a unit sphere $U_0$ centered at $\mathbf 0$ \cite{gao18information}.
We define the smooth approximation of a function $g (\mathbf x)$ with smoothing parameter $\muR$ as $ g_{\muR} (\mathbf x) = \mathbb{E}_{\mathbf{u} \in U_0} [g(\mathbf{x} + \muR \mathbf{u} )]$.
The rationale behind RGE \eqref{eq_RGE} is that $\Hat{\nabla}_{\mathrm{RGE}} F (\mathbf x; \xi)$ is an \textit{unbiased} estimate of $\nabla F_{\muR} (\mathbf x;\xi)$, leading to
$ \mathbb{E}_{\xi, \{ \mathbf u_i \}}[\Hat{\nabla}_{\mathrm{RGE}} F (\mathbf x; \xi)] = \nabla f_{\muR} (\mathbf x)$. Note that
$\nabla f_{\muR} (\mathbf x)$ itself is a biased approximation of the true gradient $\nabla f (\mathbf x)$, with the bias controlled by $\muR$.

\paragraph{CGE.}
Different from RGE, CGE is constructed by finite difference approximations of directional derivatives along the canonical basis vectors $\{ \mathbf e_i \}_{i=1}^d$ in $\mathbb{R}^d$. 
The $i$-th component of CGE is defined as
\cite{kiefer1952stochastic}

    \begin{align} \label{eq_CGE_1}
	    \Hat{\nabla}_{\mathrm{CGE}} F_i (\mathbf{x}; \xi)  =  \frac{ \left [ F(\mathbf{x} + \muCi \mathbf{e}_{i}; \xi) - F(\mathbf{x} - \muCi \mathbf{e}_{i}; \xi) \right ]}{2 \muCi} \mathbf{e}_{i}, 
    \end{align}
where $\muCi > 0$ denotes the coordinate-wise smoothing parameter. 
$ \Hat{\nabla}_{\mathrm{CGE}} F (\mathbf{x}; \xi)  =  \sum_{i=1}^d \Hat{\nabla}_{\mathrm{CGE}} F_i (\mathbf{x}; \xi)$ is the \textit{full} CGE.
If a random subset of the coordinates $\mathcal{I} \subseteq [d]$ is used (here $[d]$ denotes the set $\{ 1,2,\ldots, d\}$), rather than the full coordinate set $[d]$, we get the stochastic coordinate-wise gradient estimator   \cite{lian2016comprehensive,gu2016asynchronous}
\begin{align} \label{eq_CGE_I}
    \Hat{\nabla}_{\mathrm{CGE}} F_{\mathcal I} (\mathbf{x}; \xi) = \frac{d}{\nC} \sum_{i=1}^d I(i \in \mathcal{I}) \Hat{\nabla}_{\mathrm{CGE}} F_i (\mathbf{x}; \xi), 
\end{align}
where the cardinality of $\mathcal I$ is denoted by $|\mathcal{I}| = \nC$,
and $I(i \in \mathcal{I})$ is the indicator function which takes the value $1$ if $i \in \mathcal{I}$  and $0$ otherwise. 
Note that \cite{lian2016comprehensive} sampled the elements of $\mathcal{I}$ \textit{uniformly}, i.e., $\mathrm{Pr} (i \in \mathcal{I}) = \nC/d$, for all $i \in [d]$. 
Hence, the multiplicative factor $d/\nC$ in \eqref{eq_CGE_I} ensures unbiasedness $ \mathbb{E}_{\mathcal{I}} [\Hat{\nabla}_{\mathrm{CGE}} F_{\mathcal I} (\mathbf{x}; \xi)] = \Hat{\nabla}_{\mathrm{CGE}} F (\mathbf{x}; \xi)$. 
In  this work, we generalize \eqref{eq_CGE_I} to the case where the coordinates are instead, sampled \textit{non-uniformly}.

Using the gradient estimate $\Hat{\nabla} F (\mathbf{x}; \xi) $ based on either RGE or CGE, we can further define the approximation of the stochastic gradient of the objective in \eqref{eq_problem}, over a set of stochastic samples $\{ \xi \in \mathcal B \}$. This leads to 
\begin{align}
\Hat{\nabla} F (\mathbf x; \mathcal B)    = \frac{1}{|\mathcal B|} \sum_{ \xi  \in \mathcal B}\Hat{\nabla} F (\mathbf{x}; \xi).
\label{eq: GE_stoc_batch}
\end{align}
We also remark that
compared to  RGE, full CGE takes $O(d/\nR)$ times more function queries if $d > \nR$. 
However, it has $O(d/\nR)$ times smaller gradient estimation error  \cite{liu2018zeroth}. 
Inspired by this observation, we ask:

\noindent \textit{Can a well-designed convex combination of RGE and CGE improve the gradient estimation accuracy over RGE as well as improve the query efficiency over CGE?}

\section{Hybrid Gradient Estimator}
\label{sec_HGE}
Spurred by the capabilities and limitations of RGE and CGE, in what follows we  propose a new hybrid gradient estimator (HGE) and its   variance-controlled version using a coordinate importance sampling method.  

\paragraph{Proposed HGE.}
In order to achieve the desired tradeoff between estimation accuracy and query efficiency, we combine RGE with CGE to obtain HGE
\begin{align}
   \Hat{\nabla}_{\mathrm{HGE}} F(\mathbf x; \br,  \bc, \mathcal I) =  \alpha 
   \Hat{\nabla}_{\mathrm{RGE}} F (\mathbf x; \br) + (1-\alpha) 
  \Hat{\nabla}_{\mathrm{CGE}} F_{\mathcal{I}} (\mathbf{x}; \bc, \mathbf{p}), \label{eq_HGE}
\end{align}
where  $\alpha \in [0, 1]$ is the combination coefficient, $\br$ and $\bc$ are mini-batches of stochastic samples as introduced   in \eqref{eq: GE_stoc_batch}, $ \Hat{\nabla}_{\mathrm{RGE}} F (\mathbf x; \br)$ is given by \eqref{eq_RGE} and \eqref{eq: GE_stoc_batch}, 
and $\Hat{\nabla}_{\mathrm{CGE}} F_{\mathcal{I}} (\mathbf{x}; \bc, \mathbf{p})$ denotes the stochastic CGE \eqref{eq_CGE_I} with the coordinate selection probability $\mathrm{Pr} (i \in \mathcal{I}) = p_i$, namely,
\begin{align}
 &\Hat{\nabla}_{\mathrm{CGE}} F_{\mathcal I} (\mathbf{x};  \xi, \mathbf p) = \sum_{i=1}^d \frac{I(i \in \mathcal{I})}{p_i}  \Hat{\nabla}_{\mathrm{CGE}} F_i (\mathbf{x}; \xi). \label{eq_CGE_nonuniform}
\end{align}
%
%
Different from \eqref{eq_CGE_I}, the coordinates of $\mathcal{I}$ in \eqref{eq_CGE_nonuniform}  can be sampled with unequal probabilities. It can be easily shown that 
$ \mathbb{E} [ \Hat{\nabla}_{\mathrm{CGE}} F_{\mathcal{I}} (\mathbf{x}; \bc, \mathbf{p})] = \Hat{\nabla}_{\mathrm{CGE}} f(\mathbf{x})$.
In \eqref{eq_HGE},
$\alpha$ and $\mathbf{p}$  are design parameters.
Their optimal values help us improve the variance of HGE relative to RGE, and the function query complexity relative to CGE.
We next discuss how to optimize these parameters.
 
\paragraph{Design of coordinate selection probabilities $\mathbf{p}$.} 
Recall that if no prior knowledge on gradient estimation is available, then a simple choice of $\mathbf p$ is that of equal probability across all coordinates \cite{lian2016comprehensive}, namely, the uniform distribution  used in \eqref{eq_CGE_I}.
However in HGE, a RGE $\mathbf{g} \triangleq \Hat{\nabla}_{\mathrm{RGE}} F (\mathbf x; \br)$, known prior to computing CGE, can be employed as a probe estimate. 
Thus, we ask if $\mathbf p$ could be designed based on $\mathbf g$ to reflect the importance of   coordinates to be selected.
We next show that this question can be addressed by formulating the coordinate selection problem as the problem of gradient sparsification \cite{wangni2018gradient}.


Given $\mathbf p$,  the $i$-th coordinate of $\mathbf{g}$ is   dropped with
   probability $1-p_i$.
This can be modeled using a Bernoulli random variable $Z_i$: $\mathrm{Pr} (Z_i = 1) = p_i$ and $ \mathrm{Pr} (Z_i = 0) = 1-p_i$. 
Defining $(Q(\mathbf{g}))_i = Z_i \cdot (g_i/p_i)$, note that $Q(\mathbf{g})$ is an unbiased estimator of $\mathbf{g}$.
The variance can be bounded using $ \lnr  Q(\mathbf{g}) \rnr^2 = \sum_{i=1}^d g_{i}^2/p_i$, while the expected sparsity of $Q(\mathbf{g})$ is $\sum_{i=1}^d p_i$.
To determine $\mathbf{p}$, we minimize the variance of the sparsified RGE under a constraint on the expected sparsity of the vector. That is, we solve the problem:
 \begin{align}
	\min_{\mathbf p} \sum_{i=1}^d \frac{g_{i}^2}{p_i} \ \quad \text{subject to} \quad \ \sum_{i=1}^d p_i \leq \nC, 0 < p_i \leq 1, \forall i. \label{eq_prob_opt}
\end{align}
where $\nC \in \mathbb{N}_+$ denotes the coordinate selection budget, and $\mathbb{N}_+$ is the set of positive integers. The solution to \eqref{eq_prob_opt} is given by the following proposition.

\begin{prop}
\label{prop_prob_opt}
Suppose we denote by $g_{(1)}, g_{(2)}, \hdots, g_{(d)}$ the components of vector $\mathbf{g}$, arranged in descending order of magnitudes. First, we find the smallest $k$ such that
\begin{align*}
	| g_{(k+1)} | \left( \nC-k \right) \leq \sum_{i=k+1}^d | g_{(i)} |,
\end{align*}
is true, and denote by $S_k$ the set of coordinates with the top $k$ largest magnitudes of $| g_i |$. Then the $i$-th component of the probability vector $\mathbf{p}$ is computed as

	\begin{align*}
		p_{i} = 
		\begin{cases}
			1, & \text{ if } i \in S_k \\
			\frac{| g_{i} | (\nC-k)}{\sum_{j = k+1}^d | g_{j} |}, & \text{ if } i \notin S_k.
		\end{cases}
	\end{align*}
 
\end{prop}

\begin{proof}
See Appendix \ref{proof_prop_prob_opt}.
\end{proof}

The probability value $p_i$ depends on the relative magnitude of $g_i$, with respect to the other elements of $\mathbf{g}$.
If $\nC$ is large enough, then the coordinates corresponding to the largest elements are always sampled $(k>0)$. In the extreme case of all entries having the same magnitude (for example, $\mathbf{g}$ is the $1$-bit compressed version of a real-valued vector), $k=0$ and $p_i = \nC/d$ for all $i$.
The probabilities $\{ p_i \}$ obtained in Proposition \ref{prop_prob_opt} are used to compute \eqref{eq_CGE_nonuniform}.

\paragraph{Design of  combination coefficient $\alpha$.}
Once we select $\mathbf{p}$, we intend to select a combination coefficient $\alpha$ which can minimize the variance of HGE.
In fact, the closed form expression of this variance is not tractable.
Hence, we first upper bound the variance in the following proposition.
Then, $\alpha$ is selected to minimize this upper bound.

\begin{prop}
\label{prop_bd_var_RGE_CGE}
The variance of HGE \eqref{eq_HGE} is bounded as
\begin{align}
    \mathbb{E} \lnr \Hat{\nabla}_{\mathrm{HGE}} F(\mathbf x; \br, \bc, \mathcal{I}) - \nabla f (\mathbf{x})  \rnr ^2 & \leq 2 \alpha^2 \mathbb{E} \lnr \Hat{\nabla}_{\mathrm{RGE}} F (\mathbf x; \br) - \nabla f (\mathbf{x})  \rnr ^2 \nonumber \\
    & \quad + 2 (1-\alpha)^2 \mathbb{E} \lnr   [\Hat{\nabla}_{\mathrm{CGE}} F_{\mathcal{I}} (\mathbf{x}; \bc, \mathbf{p})] - \nabla f (\mathbf{x})  \rnr ^2. \label{eq_var_HGE}
\end{align}
Moreover, suppose Assumption \ref{assum_lipschitz} and \ref{assum_var_bound} hold. Given the probability vector $\mathbf{p}$, the individual gradient estimators $\Hat{\nabla}_{\mathrm{RGE}} F (\mathbf x; \br)$ and $\Hat{\nabla}_{\mathrm{CGE}} F_{\mathcal{I}} (\mathbf{x}; \bc, \mathbf{p})$ satisfy
\begin{align}
	& \mathbb{E} \lnr \Hat{\nabla}_{\mathrm{RGE}} F (\mathbf x; \br) - \nabla f (\mathbf{x})  \rnr ^2 \leq \frac{2}{|\br|} \left( 1 + \frac{d}{\nR} \right) \left\| \nabla f (\mathbf{x}) \right\|^2 + \frac{2 \sigma^2}{|\br|} \left( 1 + \frac{d}{\nR} \right) + \left( 1 + \frac{2}{|\br|} + \frac{2}{\nR |\br|} \right) \frac{\muR^2 L^2 d^2}{4}, \label{eq_bd_var_RGE} \\
	& \mathbb{E} \lnr [\Hat{\nabla}_{\mathrm{CGE}} F_{\mathcal{I}} (\mathbf{x}; \bc, \mathbf{p})] - \nabla f (\mathbf{x})  \rnr ^2 \leq \sum_{i=1}^d \frac{1}{p_i} \Big[ 2 \left( \nabla f (\mathbf{x}) \right)_i^2 + \frac{3}{|\bc|} \left( \zeta^2 + \frac{L^2 \muCi^2}{2} \right) + \frac{L^2 \muCi^2}{2} \Big] - 2 \left\| \nabla f (\mathbf{x}) \right\|^2, \label{eq_bd_var_CGE}
\end{align}
where recall from \eqref{eq_HGE} that $\bc$ and $\br$ denote the sets of stochastic samples, $\nR$ denotes the number of random directions (per stochastic sample) used in computing $\Hat{\nabla}_{\mathrm{RGE}}$ \eqref{eq_RGE}, $\muR$ is the RGE smoothing parameter,   $\{ \muCi \}_i$ are the coordinate-wise CGE smoothing parameters, and $\zeta$ is the coordinate-wise variance (Assumption \ref{assum_var_bound}), with $\sigma^2 = d \zeta^2$.
\end{prop}

\begin{proof}
See Appendix \ref{proof_prop_bd_var_RGE_CGE}.
\end{proof}

The accuracy of $\Hat{\nabla}_{\mathrm{RGE}}$ depends critically on the number of random directions $\nR$ \eqref{eq_RGE}. 
Also, for $|\br| \to \infty$, $\Hat{\nabla}_{\mathrm{RGE}} \to \nabla f_{\muR}$, and the bound \eqref{eq_bd_var_RGE} reduces to $O(\muR^2 L^2 d^2)$. This is precisely the bound for the deterministic case
\cite[Lemma~4.1]{gao18information}.
Similarly, the accuracy of $\Hat{\nabla}_{\mathrm{CGE}}$ depends on the sampling probabilities $\{ p_{i} \}$. 
If $p_i = 1$ for all $i$, and $|\bc| \to \infty$, $\Hat{\nabla}_{\mathrm{CGE}} F_{\mathcal{I}} (\mathbf{x}; \bc, \mathbf{p}) \to \Hat{\nabla}_{\mathrm{CGE}} f(\mathbf{x})$.
The bound in \eqref{eq_bd_var_CGE} then reduces to the deterministic case upper bound $O(L^2 \sum_{i=1}^d \muCi^2)$ \cite[Lemma~3]{zo_spider_liang19icml}.

Substituting \eqref{eq_bd_var_RGE}, \eqref{eq_bd_var_CGE} in \eqref{eq_var_HGE}, and minimizing over $\alpha$, we obtain the optimal value $\alpha^*$. 
We define $\bar{P} = \frac{1}{d} \sum_{i=1}^d \frac{1}{p_i}$. On simplification (see Appendix \ref{app_alpha_choice}), $\alpha^*$ reduces to
\begin{align}\label{eq: alpha}
    \alpha^* = \left[ 1 + \frac{(1 + d/\nR)}{\bar{P}} \right]^{-1},
\end{align}
Recall that $\nR$ and $\sum_i p_i (= \nC)$ respectively, are the function query budgets of RGE and CGE. 
Considering some extreme cases, if $\nR \to \infty$, since $\bar{P} \geq 1$, we get $\alpha^* \geq \frac{1}{2}$.
If $\bar{P} \to \infty$, then $\alpha^* \to 1$ and RGE dominates the estimator.
On the other hand, if $\nR = 0$, then $\alpha^* = 0$ since there is no RGE to assign any weight.
The relative values of $\nR$, $\bar{P}$ determine the exact weights assigned to RGE, CGE in \eqref{eq_HGE}.

A relatively simple case is the one with uniform sampling of coordinates, i.e., $p_i = \nC/d$, for all $i$.
In this case, $\nC$ is the query budget of CGE, and $\alpha^* = ( 1 + \nC/d + \nC/\nR )^{-1}$. 
If $\nC \geq \nR$, then $\alpha^* < \frac{1}{2}$ and HGE \eqref{eq_HGE} reasonably assigns higher weight to CGE.


\section{ZO Hybrid Gradient Descent}
\label{sec_HGD}
In this section, we introduce the ZO hybrid gradient descent (ZO-HGD) algorithm to solve problem \eqref{eq_problem} with the aid of HGE \eqref{eq_HGE}.
We then derive its convergence rate and discuss the performance of ZO-HGD in several special cases. 



\begin{algorithm}[h!]
\caption{{ZO-HGD to solve problem \eqref{eq_problem}}}
\label{Algo_zo_hgd}
\begin{algorithmic}[1]
	\State{\textbf{Input}: Initial point $\mathbf{x}_0$, number of iterations $T$, step sizes $\{ \eta_t \}$, smoothing parameter $\muR$ and number of random directions $\nR$ for RGE, smoothing parameters $\{ \muCi \}_{i=1}^d$ for CGE, random sample set $\Xi$, coordinate selection budget $\{ n_{\mathrm{c,t}} \}$ for CGE \eqref{eq_prob_opt} 
	}
	\For{$t = 0$ to $T-1$}
    	\State Sample mini-batches $\brt, \bct$ from $\Xi$
    	\State Compute RGE $\nabla_{\mathrm{r},t} = \Hat{\nabla}_{\mathrm{RGE}} F (\mathbf{x}_t; \brt)$ as \eqref{eq_RGE} \& \eqref{eq: GE_stoc_batch} \label{line_rge}
        \State Compute importance sampling probabilities $\mathbf p_t$ with $\nC = n_{\mathrm{c,t}}, \mathbf{g} = \nabla_{\mathrm{r},t}$ 
        in Proposition \ref{prop_prob_opt} 
        \label{line_prob}
        \State Sample coordinate set $\mathcal{I}_t$ of cardinality $n_{\mathrm{c,t}}$ with $\mathrm{Pr} (i \in \mathcal{I}_t) = p_{t,i}$ \label{line_cge1}
        \State Compute CGE $\nabla_{\mathrm{c},t} =   \Hat{\nabla}_{\mathrm{CGE}} F_{\mathcal{I}_t} (\mathbf{x}_t; \bct, \mathbf{p}_t)$ as \eqref{eq_CGE_nonuniform} \label{line_cge}
        \State Update: $\mathbf{x}_{t+1} = \mathbf{x}_{t} - \eta_t \left( \alpha_t \nabla_{\mathrm{r},t} + (1 - \alpha_t) \nabla_{\mathrm{c},t} \right)$ \label{line_update}
	\EndFor
	\State \textbf{Output:} A solution $\bar{\mathbf{x}}_T$ is  picked uniformly randomly from $\{ \mathbf{x}_t \}_{t=1}^{T}$.
\end{algorithmic}
\end{algorithm}

We present ZO-HGD in Algorithm\,\ref{Algo_zo_hgd}, which consists of three main steps. 
First, a coarse gradient estimate  $\nabla_{\mathrm{r},t}$ is acquired   using RGE (line \ref{line_rge}), and is employed as a probe signal to determine the probabilities  $\mathbf p_t$ of coordinate-wise importance sampling (line \ref{line_prob}).
Second, a stochastic CGE is generated using the subset of coordinates $\mathcal{I}_t$ sampled according to $\mathbf p_t$ (line \ref{line_cge1}-\ref{line_cge}).
Third, a HGE based descent step is used to update the optimization variable $\mathbf x_t$ per iteration (line \ref{line_update}).
 
The expected total number of function evaluations in Algorithm \ref{Algo_zo_hgd}, known as the \textit{Function Query Cost} (FQC), is given by $\sum_{t=0}^{T-1} ( 2\nR|\brt| + 2 n_{\mathrm{c,t}} |\bct| )$, where recall that $n_{\mathrm{c,t}}$ is the size of the coordinate set $\mathcal{I}_t$ sampled in line \ref{line_cge1}.
%
%
If we assume the coordinate set size to be constant over time, i.e., $n_{\mathrm{c,t}} = \nC$ at all times $t$, FQC reduces to $O ( T (\nR+ \nC) )$. 
Note that if $n_{\mathrm{c,t}} = 0$, 
then HGE reduces to RGE, with $\alpha_t = 1$. Accordingly, Algorithm \ref{Algo_zo_hgd} becomes mini-batch ZO-SGD \cite{gao18information}. 
On the other hand, if $\nR=0$, HGE reduced to CGE and $\alpha_t = 0$. 
In this case, no prior knowledge is available to compute the sampling probability vector $\mathbf{p}_t$. 
If we sample the coordinates uniformly, i.e., $p_{t,i} = \nC/d$, for all $i,t$, ZO-HGD reduces to ZO-SCD \cite{lian2016comprehensive}.


\paragraph{Technical challenges of ZO-HGD.} 
ZO-HGD is a non-trivial extension of both ZO-SGD and ZO-SCD because of the following differences. 
First, the \textit{non-uniform} sampling of the coordinate sets $\{ \mathcal{I}_t \}$ in CGE complicates the derivation of the upper bound on the variance of CGE \eqref{eq_bd_var_CGE}, compared to the uniform sampling case \cite{lian2016comprehensive}. 
Second, as discussed in Section \ref{sec_HGE}, the choice of the \textit{combination coefficients} $\{ \alpha_t \}$  controls the variance of the proposed HGE relative to both RGE and CGE.
However, $\alpha$ has a nonlinear dependence on $\nR$ (query budget for RGE), and $\{ p_{t,i} \}$ (sampling probabilities for CGE).
This makes choosing $\alpha_t$   nontrivial {to achieve the graceful tradeoff between convergence speed and FQC.}
Next, we elaborate on the convergence of ZO-HGD.




\paragraph{Convergence analysis.}
For ease of notation, as in Algorithm \ref{Algo_zo_hgd}, we denote $$\nabla_{\mathrm{r},t} = \Hat{\nabla}_{\mathrm{RGE}} F (\mathbf{x}_t; \brt),
\nabla_{\mathrm{c},t} =   \Hat{\nabla}_{\mathrm{CGE}} F_{\mathcal{I}_t} (\mathbf{x}_t; \bct, \mathbf{p}_t).$$%
Also, we assume the coordinate-wise smoothing parameters in CGE to be fixed, i.e., $\muCi = \muC$ for all $i$.
Recall that the function query budget of HGE depends only on $\nR$ (number of random vectors in RGE), and the sampling probability values $\{ p_{t,i} \}$ for CGE. The expected number of coordinates used in CGE at time $t$ is $\sum_{i=1}^d p_{t,i} = \nC$.

Our analysis begins with using the L-smoothness of $f$ (Assumption \ref{assum_lipschitz}) in the update step (line \ref{line_update}) in Algorithm \ref{Algo_zo_hgd},
\begin{align}
	f (\mathbf{x}_{t+1}) & \leq f (\mathbf{x}_{t}) - \eta_t \left\langle \nabla f (\mathbf{x}_{t}), \alpha_t \nabla_{\mathrm{r},t} + (1-\alpha_t) \nabla_{\mathrm{c},t} \right\rangle + \frac{\eta_t^2 L}{2} \left\| \alpha_t \nabla_{\mathrm{r},t} + (1-\alpha_t) \nabla_{\mathrm{c},t} \right\|^2. \label{eq_f_smoothness}
\end{align}
We denote by $Y_t \triangleq \{ \mathcal{I}_t, \bct, \brt \}$ 
the randomness at step $t$.
We denote by $\mathcal{Y}_t$ the $\sigma$-algebra generated by $\{ Y_0, Y_1, \hdots, Y_{t-1} \}$. 
Taking conditional expectation in \eqref{eq_f_smoothness},
\begin{align}
	\mathbb{E}_{Y_t} \left[ f (\mathbf{x}_{t+1}) \mid \mathcal{Y}_t \right] & \leq f (\mathbf{x}_{t}) + \eta_t \alpha_t \underbracket[0.8pt]{\left\langle -\nabla f (\mathbf{x}_{t}), \nabla f_{\muR} (\mathbf{x}_{t}) \right\rangle}_\text{\clap{I~}} +\eta_t (1 - \alpha_t) \underbracket[0.8pt]{\left\langle -\nabla f (\mathbf{x}_{t}), \Hat{\nabla}_{\mathrm{CGE}} f (\mathbf{x}_{t}) \right\rangle}_\text{\clap{II~}} \label{eq_f_smoothness_exp1} \\
	& + \eta_t^2 L \alpha_t^2  \underbracket[0.8pt]{\mathbb{E}_{Y_t} [ \left\| \nabla_{\mathrm{r},t} \right\|^2 \mid \mathcal{Y}_t ]}_\text{\clap{III~}} + \eta_t^2 L (1 - \alpha_t)^2 \underbracket[0.8pt]{\mathbb{E}_{Y_t} [ \left\| \nabla_{\mathrm{c},t} \right\|^2 \mid \mathcal{Y}_t ]}_\text{\clap{IV~}}, \nonumber
\end{align}
which holds since \textit{(i)} $\mathbb{E}_{Y_t} \left[ \nabla_{\mathrm{r},t} \mid \mathcal{Y}_t \right] = \nabla f_{\muR} (\mathbf{x}_{t})$, where recall from Section \ref{sec_prelim} that $f_{\muR}$ is a smooth approximation of $f$; 
\textit{(ii)} $\mathbb{E}_{Y_t} \left[ \nabla_{\mathrm{c},t} \mid \mathcal{Y}_t \right] = \Hat{\nabla}_{\mathrm{CGE}} f (\mathbf{x}_{t})$, which follows from \eqref{eq_CGE_I}, where $\Hat{\nabla}_{\mathrm{CGE}} f$ is the full-coordinate CGE of $f$;
\textit{(iii)} $ \lnr  \mathbf{a} + \mathbf{b}  \rnr ^2 \leq 2 \lnr \mathbf{a}  \rnr ^2 + 2 \lnr \mathbf{b}  \rnr ^2$.
Next, we bound the  terms I-IV of \eqref{eq_f_smoothness_exp1} in Proposition \ref{prop_inner_prod}, \ref{prop_bd_norm_sq_rge}. We begin with I, II.

\begin{prop}
\label{prop_inner_prod}
Suppose $f$ satisfies Assumption \ref{assum_lipschitz}, then the quantities I and II in \eqref{eq_f_smoothness_exp1} are upper bounded as
\begin{align}
    \mathrm{I} 
    \leq -\frac{3}{4} \left\| \nabla f (\mathbf{x}_{t}) \right\|^2 + \frac{(\muR d L)^2}{4}, 
    \qquad \mathrm{II}  
    \leq -\frac{3}{4} \left\| \nabla f (\mathbf{x}_{t}) \right\|^2 + L^2 d \muC^2,\nonumber
\end{align}
where recall that 
$\muR$ and $\muC$ are the smoothing parameters, respectively for RGE and CGE.
\end{prop}

\begin{proof}
See Appendix \ref{proof_prop_inner_prod}.
\end{proof}

\noindent In first-order problems \cite{beck17first_book}, given an unbiased estimator ($\tilde{\nabla} f$) of the gradient $\nabla f$, we simply get $\mathbb{E} \langle \nabla f (\mathbf{x}), \tilde{\nabla} f (\mathbf{x}) \rangle = \lnr \nabla f (\mathbf{x})  \rnr ^2$. 
However, $\nabla_{\mathrm{r},t}$ (RGE) and $\nabla_{\mathrm{c},t}$ (CGE) in Algorithm \ref{Algo_zo_hgd} could be biased estimates of $\nabla f$. 
Proposition \ref{prop_inner_prod} bounds the deviation of the inner products I, II from $ - \lnr  \nabla f (\mathbf{x})  \rnr ^2$, in terms of the respective smoothness parameters $\muC, \muR$. 
We next bound III, IV in \eqref{eq_f_smoothness_exp1}.

\begin{prop}
\label{prop_bd_norm_sq_rge}
Suppose $f$ satisfies Assumption \ref{assum_lipschitz}, \ref{assum_var_bound}, then the quantities III and IV in \eqref{eq_f_smoothness_exp1} are upper bounded as
\begin{align}
	\mathrm{III} &= \mathbb{E}_{Y_t} [ \left\| \nabla_{\mathrm{r},t} \right\|^2 \mid \mathcal{Y}_t ] \leq \left[ 2 + \frac{4}{|\br|} \left( 1 + \frac{d}{\nR} \right) \right] \left\| \nabla f (\mathbf{x}) \right\|^2 + \frac{4 \sigma^2}{|\br|} \left( 1 + \frac{d}{\nR} \right) + \left( 1 + \frac{2}{|\br|} + \frac{2}{\nR |\br|} \right) \frac{\muR^2 L^2 d^2}{2}, \nonumber \\
	\mathrm{IV} &= \mathbb{E}_{Y_t} [ \left\| \nabla_{\mathrm{c},t} \right\|^2 \mid \mathcal{Y}_t ] \leq \sum_{i=1}^d \frac{1}{p_{t,i}} \Big[ 2 \left( \nabla f (\mathbf{x}_{t}) \right)_i^2 + \frac{3 \zeta^2}{|\bct|} + \frac{L^2 \muC^2}{2} \left( 1 + \frac{3}{|\bct|} \right) \Big]. \nonumber
\end{align}
\end{prop}

\begin{proof}
See Appendix \ref{proof_prop_bd_norm_sq_rge}.
\end{proof}
\noindent The bounds in Proposition \ref{prop_bd_norm_sq_rge} are expressed in terms of the gradient norm, variance $\sigma^2$, and the parameters which characterize RGE and CGE: number of random directions $\nR$, importance sampling probabilities $\{ p_{t,i} \}$, batch-sizes $|\bc| $ and $ |\br|$, and the smoothness parameters $\muR$ and  $\muC$.
If $|\br| \to \infty$, then $\Hat{\nabla}_{\mathrm{RGE}} \to \nabla f_{\muR}$ as seen in Sec.\,\ref{sec_HGE}, and the upper bound in Proposition \ref{prop_bd_norm_sq_rge} reduces to $2 \mathbb{E} \left\| \nabla f (\mathbf{x}) \right\|^2 + \muR^2 L^2 d^2/2$. 
This is precisely the bound acquired using the exact ZO gradient estimator $\nabla f_{\muR}$ \cite[Lemma~1]{liu2018zeroth}. Here
$\muR^2 L^2 d^2/2$ quantifies the inevitable bias due to the use of the ZO gradient estimator $\nabla f_{\muR}$. 

Now that we have bounded all the terms in \eqref{eq_f_smoothness_exp1}, we are ready to present the main result of this section. We define $\bar{P}_T = \frac{1}{dT} \sum_{t=0}^{T-1} \sum_{i=1}^d \frac{1}{p_{t,i}}$, and $\dnr = 1 + d/\nR$.

\begin{theorem}
\label{thm_nonconvex}
Suppose Assumption \ref{assum_lipschitz} and \ref{assum_var_bound} hold, and the set of coordinate-wise probabilities $\{ p_{t,i} \}$ satisfy $p_{t,i} \geq c_t > 0$, for all $i, t$. 
We choose 
the smoothing parameters $\muC, \muR$ such that $\muC = (\muR \sqrt{d})/2$ and $\muC L \sqrt{d} \leq \sigma$, where $\sigma$ is the variance.
We take constant step sizes $\eta_t = \eta \leq \frac{1}{24 L} \min \{ 3 c_t, 1/\dnr \}$ for all $t$, and combination coefficients $\alpha_t = \alpha = \left[ 1 + (\dnr/\bar{P}_T) \right]^{-1}$ for all $t$. 
%
With $\muC = \mathcal{O} \left( ( \dnr/(d^2 T))^{1/4} ( 1 + \dnr/\bar{P}_T )^{-1/4} \right)$, we obtain
 \begin{align}
	\mathbb{E} \left\| \nabla f (\bar{\mathbf{x}}_T) \right\|^2 \leq \mathcal{O} \left( \sqrt{ \frac{\dnr}{T} \frac{1}{1 + \dnr/\bar{P}_T} } \right). \label{eq_thm_nonconvex}
\end{align}
\end{theorem}

\begin{proof}
See Appendix \ref{proof_thm_nonconvex}.
\end{proof}

We compare ZO-HGD with several existing methods in Table \ref{table_comparison}, in terms of the allowed smoothing parameter values, convergence rate and FQC.
As illustrated, the smoothing parameters values
for ZO-HGD are less restrictive than other ZO algorithms. We present more insights into Theorem\,\ref{thm_nonconvex} in the next subsection.

\begin{table}[ht]
\begin{center}
\begin{tabular}{|c|c|c|c|}
\cline{1-4}
Method & \begin{tabular}[c]{@{}c@{}} Smoothing\\parameter \end{tabular} & \begin{tabular}[c]{@{}c@{}} Convergence\\rate  \end{tabular} & FQC \\
\cline{1-4}
ZO-GD \cite{nesterov17grad_free} & $O \left( 1/\sqrt{d T} \right)$ & $O \left( d/T \right)$ & $O \left( |\mathcal{D}| T \right)^3$ \\
\cline{1-4}
ZO-SGD \cite{Ghadimi_Siam_2013_SGD}
 & $O \left( 1/d \sqrt{b T} \right)$ & $O \left( \sqrt{ d/(b T)} \right)$ & $O \left( T b \right)$ \\
\cline{1-4}
ZO-SCD \cite{lian2016comprehensive}
& $O \left( 1/\sqrt{b T} + (d b T)^{-1/4} \right)$ & $O \left( \sqrt{ d/(b T)} \right)$ & $O \left( T b \right)$ \\
\cline{1-4}
ZO-SVRG \cite{liu2018zeroth} & $O \left( 1/\sqrt{d T} \right)$ & $O \left( d/T + 1/b \right)$ & \begin{tabular}[c]{@{}c@{}} $O \left( |\mathcal{D}| s + b s m \right)$ \\ $T = sm$ \end{tabular} \\
\cline{1-4}
ZO-signSGD \cite{liu2019signsgd} & $O \left( 1/\sqrt{d T} \right)$ & $O \left( \sqrt{d/T} + \sqrt{d/b} \right)$ & $O \left( b T \right)$ \\
\cline{1-4}
\begin{tabular}[c]{@{}c@{}} \textbf{Our work:} $\dnr \approx \bar{P}_T$, \\ $p_{t,i} = \nC/d, \forall \ t,i$ \end{tabular} & 
$ O \left( \frac{1}{d} \left( \frac{1}{T} \left( 1 + \frac{d}{\nR + \nC} \right) \right)^{1/4} \right)$
& $O \left( \sqrt{\frac{1}{T} \left( 1 + \frac{d}{\nR + \nC} \right)} \right)$ & $O (T (\nR + \nC))$ \\
\cline{1-4}
\begin{tabular}[c]{@{}c@{}} \textbf{Our work} \\ $\dnr \gg \bar{P}_T$ \end{tabular} & $O \left( \left( \bar{P}_T/d^2 T \right)^{1/4} \right)$ & $O \left(  \sqrt{ \bar{P}_T/T} \right)$ & $O (T \nC)$ \\
\cline{1-4}
\begin{tabular}[c]{@{}c@{}} \textbf{Our work} \\ $\dnr \ll \bar{P}_T$ \end{tabular} & 
$O \left( d^{-1} \left( \dnr/T \right)^{1/4}  \right)$ & $O \left( \sqrt{\dnr/T} \right)$ & $O (T \nR)$ \\
\cline{1-4}
\end{tabular}
\caption{\footnotesize{
Comparison of different ZO algorithms in terms of smoothing parameter, convergence error, and function query cost. Here $d$ is the problem dimension, $T$ is the total number of iterations in the algorithm.
We have stated generalized results for ZO-SGD (with $b$-point gradient estimator (GE)) and ZO-SCD (with $b$ coordinate GE).
For ZO-SVRG and ZO-signSGD, $b$ is the number of random direction vectors used to compute a multi-point GE.
ZO-signSGD only converges to a neighborhood of the stationary point.
$|\mathcal{D}|$ is the size of the entire dataset.
Recall in this work, $\nR$ is the number of random direction vectors used to compute RGE, $\nC$ is the number of coordinates used to compute CGE, $\dnr = 1 + d/\nR$.}}
\label{table_comparison}
\end{center}
\end{table}

\paragraph{Tradeoff between RGE and CGE in ZO-HGD.}
Next, we consider some special cases of Theorem \ref{thm_nonconvex}, based on the values of $\{ p_{t,i} \}$ and $\nR$. 
Large $\nR$ leads to a more accurate RGE \eqref{eq_bd_var_RGE}, while large values of $\{ p_{t,i} \}$ ensure a more accurate CGE \eqref{eq_bd_var_CGE}. 
The tradeoff between RGE and CGE is captured by the ratio $\dnr / \bar{P}_T$ in \eqref{eq_thm_nonconvex}. 
Table \ref{table_comparison} highlights the performance of ZO-HGD in three regimes, when $\dnr / \bar{P}_T$ is $\ll 1, \gg 1$, and $\approx 1$. 
We discuss one of the cases in detail, the remaining follow similar reasoning. 

\subsubsection{Regime 1: $\dnr = 1 + \frac{d}{\nR} \gg \bar{P}_T$.}
Since $\bar{P}_T$ is the mean of inverse probability values $\{ 1/p_{t,i} \}$, $\frac{1}{\bar{P}_T} \gg \frac{\nR}{d}$ implies that on average, sampling probabilities are much greater than $\nR/d$. 
In other words, the per-iteration query budget of CGE ($\nC$) is much higher compared to RGE, and $\alpha \to 0$. 
From Theorem \ref{thm_nonconvex}, the smoothing parameter $\muC = O ( (\bar{P}_T/d^2 T)^{1/4} )$ leads to convergence rate  $\mathbb{E} \lnr \nabla f (\bar{\mathbf{x}}_T)  \rnr ^2 \leq O(  \sqrt{ \bar{P}_T/T } ).$


In the special case of uniform distribution for CGE, i.e., $p_{t,i} = \nC/d$, $\bar{P}_T = d/\nC$.
Consequently, the convergence rate is
 $\mathbb{E} \lnr \nabla f (\bar{\mathbf{x}}_T)  \rnr ^2 \leq O (  \sqrt{ d/(\nC T)} )$.
Also, $\frac{d}{\nR} \gg \bar{P}_T$ implies $\nC \gg \nR$.
Hence, FQC to achieve $\mathbb{E} \lnr \nabla f (\bar{\mathbf{x}}_T)  \rnr ^2 \leq \epsilon$ is  $O(T \cdot \nC) = O(d/\epsilon^2)$.
Naturally, both the convergence rate and FQC are dominated by CGE.
For $\nC = 1$ the performance reduces to that of ZO-SCD (see Table \ref{table_comparison}).

The results for the other \textbf{two regimes}: 2) when $1 + \frac{d}{\nR} \ll \bar{P}_T$, and 3) when $1 + \frac{d}{\nR}$ and $\bar{P}_T$ are comparable, and are stated in Table \ref{table_comparison}. For more details, please see Appendix \ref{app_special_case_nonconvex}.

\section{ZO-HGD for Convex Optimization}
\label{sec_convex}
We now  analyze the convergence properties of Algorithm \ref{Algo_zo_hgd} in cases where 
the objective function $f: \mathbb{R}^d \to \mathbb{R}$ is convex and strongly convex.
\begin{assump}
\label{assum_convexity}
The function $f$ is defined on a compact set and is convex, such that $f(\mathbf{y}) \geq f(\mathbf{x}) + \left\langle \nabla f(\mathbf{x}), \mathbf{y} - \mathbf{x} \right\rangle$, $\forall \ \mathbf{x}, \mathbf{y} \in \mathrm{ dom } f$.
Further, the diameter of the set $\mathrm{ dom } f$ is bounded by $R$, i.e., $ \lnr  \mathbf{y} - \mathbf{x}  \rnr  \leq R$, $\forall \ \mathbf{x}, \mathbf{y} \in \mathrm{ dom } f$. 
\end{assump}

\begin{assump}
\label{assum_bound_grad}
$\left\| \nabla f(\mathbf{x}) \right\| \leq G, \forall \ \mathbf{x} \in \mathrm{ dom } f$ for constant $G$.
\end{assump}

\begin{assump}{($\bar{\sigma}$-\textit{Strong Convexity}).}
\label{assum_strong_convex}
$\forall \ \mathbf{x}, \mathbf{y} \in \mathrm{ dom } f$, $f$ satisfies $f(\mathbf{y}) \geq f(\mathbf{x}) + \left\langle \nabla f(\mathbf{x}), \mathbf{y}-\mathbf{x} \right\rangle + \frac{\bar{\sigma}}{2} \left\| \mathbf{y}-\mathbf{x} \right\|^2$.
\end{assump}
Assumption \ref{assum_bound_grad} is fairly standard in the convex optimization literature \cite{duchi15optimal_tit, stich18spars_SGD}.
We quantify the error in terms of the average difference between the expected function values at successive iterates and the optimal function value, $(1/T) \sum_{t=1}^T \left( \mathbb{E} f(\mathbf{x}_t) - f^* \right)$. 
For brevity, we only state the theorem for our convergence result.

\begin{theorem}
\label{thm_convex}
Suppose Assumption \ref{assum_lipschitz}, \ref{assum_var_bound}, \ref{assum_convexity}, \ref{assum_bound_grad} hold and the coordinate-wise probabilities $\{ p_{t,i} \}$ satisfy $p_{t,i} \geq \bar{c} > 0$ for all $i, t$. We define $\dnr = 1 + d/\nR$.
\begin{enumerate}
    \item If the smoothing parameters $\muC, \muR$ are chosen such that $\muC = (\muR \sqrt{d})/2$, $\sigma \geq L \muC \sqrt{d}$. A suitably chosen $\alpha$ gives
    \begin{align}
	    & \frac{1}{T} \sum_{t=1}^T \mathbb{E} f(\mathbf{x}_t) - f^* \leq O \left( \sqrt{\frac{1}{T} \frac{\dnr}{ 1 + \bar{c} \dnr } } \right). \nonumber
    \end{align}
    \item If Assumption \ref{assum_strong_convex} also holds, the smoothing parameters $\muC, \muR$ are such that $\muR = \muC \sqrt{(d L)/\bar{\sigma}}$, and the step-size $\eta_t = \frac{8}{\bar{\sigma} (a+t)}$ with $a>1$. We define $\Hat{\mathbf{x}}_T = \frac{1}{S_T} \sum_{t=0}^{T-1} w_t \mathbf{x}_t$, where $w_t = (a+t)^2$, and $S_T = \sum_{t=0}^{T-1} w_t$. Then
    \begin{align}
	    \mathbb{E} f(\Hat{\mathbf{x}}_T) - f^* \leq O \left( \frac{1}{T} \frac{\dnr}{1 + \bar{c} \dnr} \right). \nonumber
    \end{align}
\end{enumerate}

\end{theorem}

\begin{proof}
See Appendix \ref{app_convex} for details.
\end{proof}


As in Section \ref{sec_HGD}, we consider the special case of uniform sampling, $p_{t,i} = \nC/d$ for all $i, t$. Let $q =\nR+ \nC$. For convex functions, the convergence rate reduces to $O( \sqrt{(1 + d/q)/T} )$.
For $O(q)$ function evaluations per iteration,
this rate is order optimal \cite{duchi15optimal_tit}.
\noindent For strongly convex functions, the convergence rate reduces to $\mathbb{E} f(\Hat{\mathbf{x}}_T) - f^* \leq O ( (1 + d/q)/T )$. Note that this latter guarantee holds for a weighted average of iterates.

%
%


\section{Experiments}
\label{sec_experiments}
In this section, we demonstrate the effectiveness of ZO-HGD via a real-world  application of generating \textit{adversarial examples} from a \textit{black-box} deep neural network (DNN) \cite{chen2017zoo}. Here  the adversarial examples are defined by    inputs
with
imperceptible perturbations crafted to mislead the DNN's prediction, and they provide a means of measuring the robustness of DNNs against adversarial perturbations \cite{Goodfellow2015explaining,carlini2017towards,xu2018structured}.

We begin by formally presenting the   problem of generating black-box adversarial examples.
Let ($\mathbf{x},t$) denote a legitimate image $\mathbf{x}$ with the true label $t$. And  $\mathbf x^\prime = \mathbf x+ \boldsymbol{\delta}$ denotes an adversarial example with the adversarial perturbation $\boldsymbol{\delta}$.
Our goal is to design $\boldsymbol{\delta}$ for misclassifying $M$ images $\{ \mathbf x_i \}_{i=1}^M$ when a  DNN is used as a decision maker. This leads to
 the optimization problem \cite{carlini2017towards}
\begin{align}\label{eq: attack_general}
\begin{array}{ll}
    \displaystyle \minimize_{\boldsymbol \delta } & \frac{\lambda}{M} \sum_{i=1}^Mf_{\boldsymbol \theta}(\mathbf x_i + \boldsymbol \delta) 
    +  \lnr \boldsymbol{\delta}  \rnr _2^2
\end{array}
\end{align}
where 
$f_{\boldsymbol \theta}(\mathbf x_i + \boldsymbol \delta)$ is specified as  the C\&W untargeted attack  loss \cite{carlini2017towards} evaluated  on a  DNN model with parameters $\boldsymbol \theta$ , and $\lambda >0$ is a regularization parameter that strikes a balance between minimizing the attack loss and the $\ell_2$  distortion. 
To solve problem \eqref{eq: attack_general}, we consider the more realistic case in which the adversary does not have access to model parameters $\boldsymbol \theta$, and thus, it  optimizes $\boldsymbol \delta$   only through function evaluations. 
Moreover, if $M = 1$  in  problem \eqref{eq: attack_general}, then the resulting solution is known as   per-image adversarial perturbation \cite{Goodfellow2015explaining}. And 
if $M > 1$, then the solution provides the \textit{universal adversarial perturbation} applied to multiple benign images simultaneously \cite{moosavi2017universal}.

\begin{figure}[htb]
\begin{center}
\begin{tabular}{c}
\includegraphics[width=.55\textwidth,height=!]{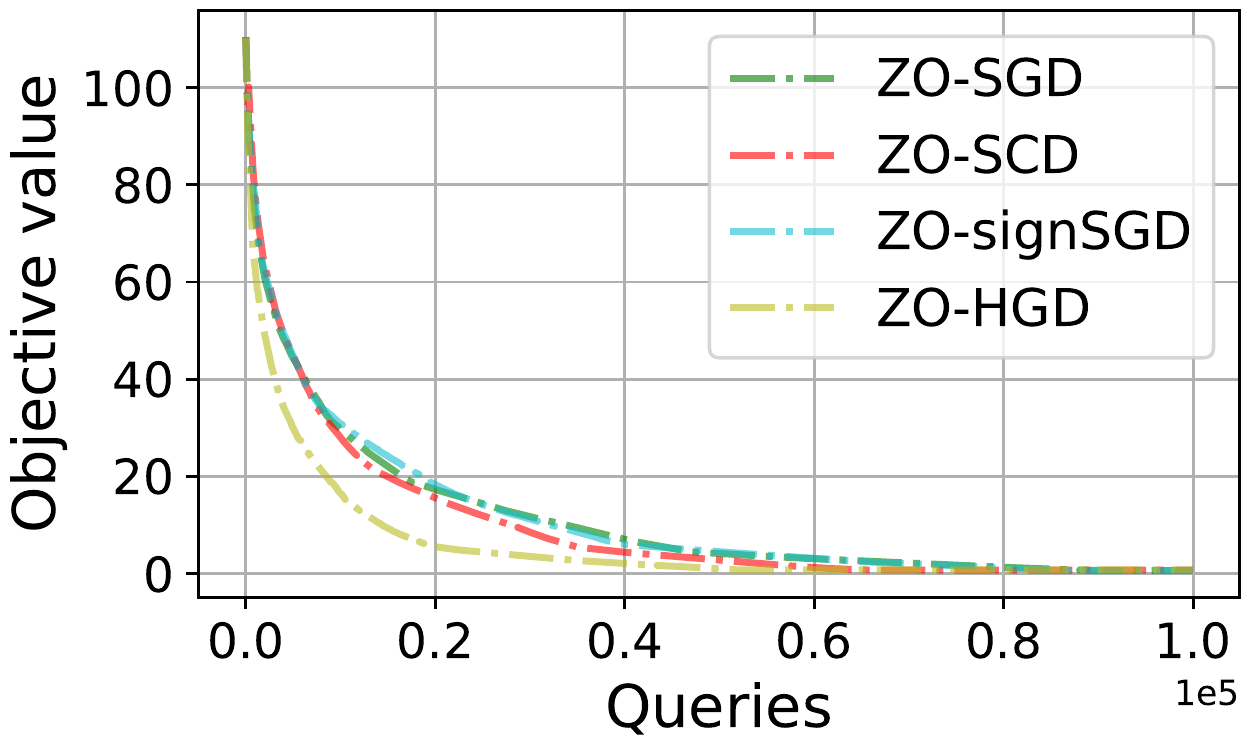} \\
\end{tabular}
\end{center}
\caption{\footnotesize{Averaged objective value when solving problem \eqref{eq: attack_general} over $50$ random trials versus the number of function queries.
}}
  \label{fig: obj_value_universal_attack}
\end{figure}

\begin{figure}[htb]
\begin{center}
\begin{tabular}{cc}
\includegraphics[width=.3\textwidth,height=!]{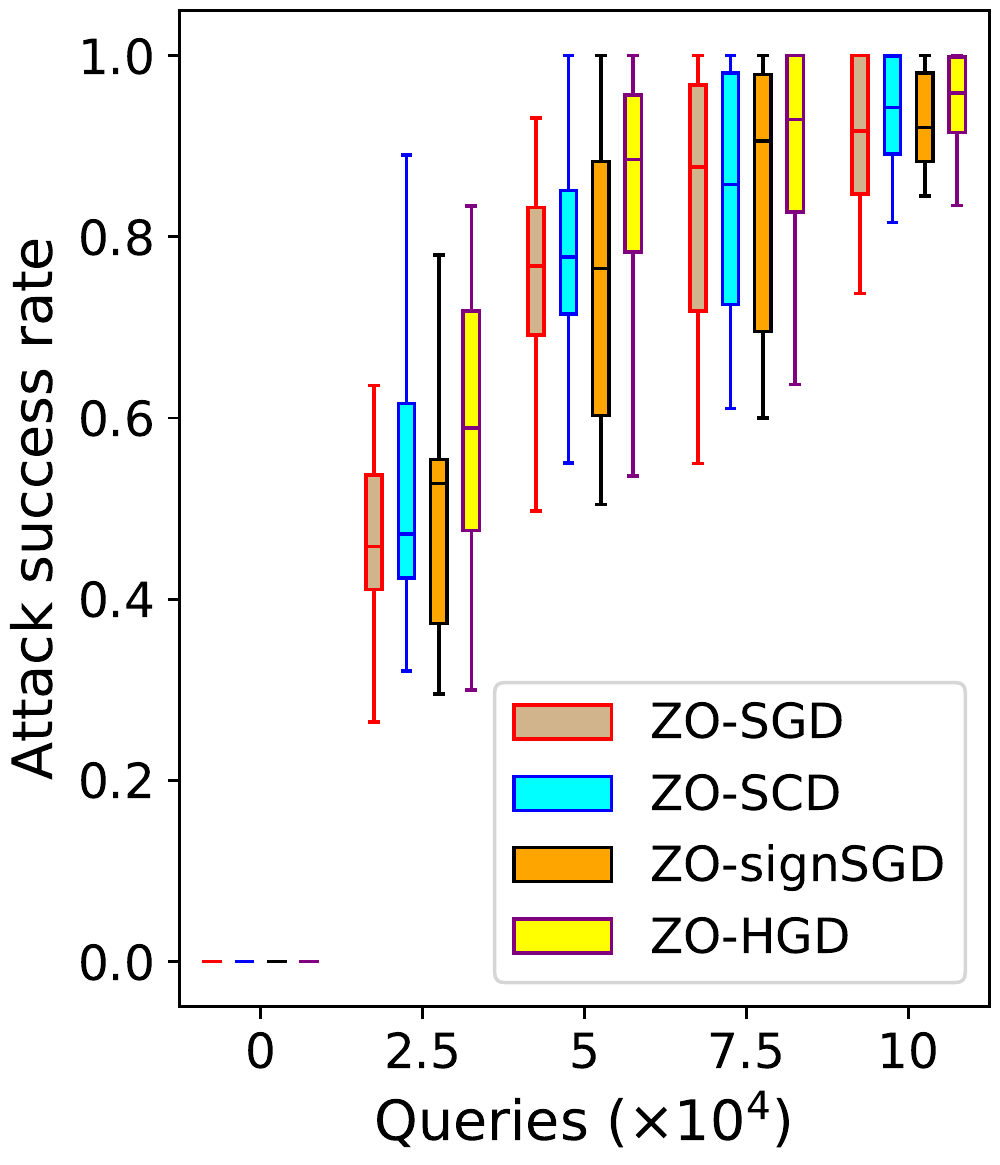}  
\includegraphics[width=.3\textwidth,height=!]{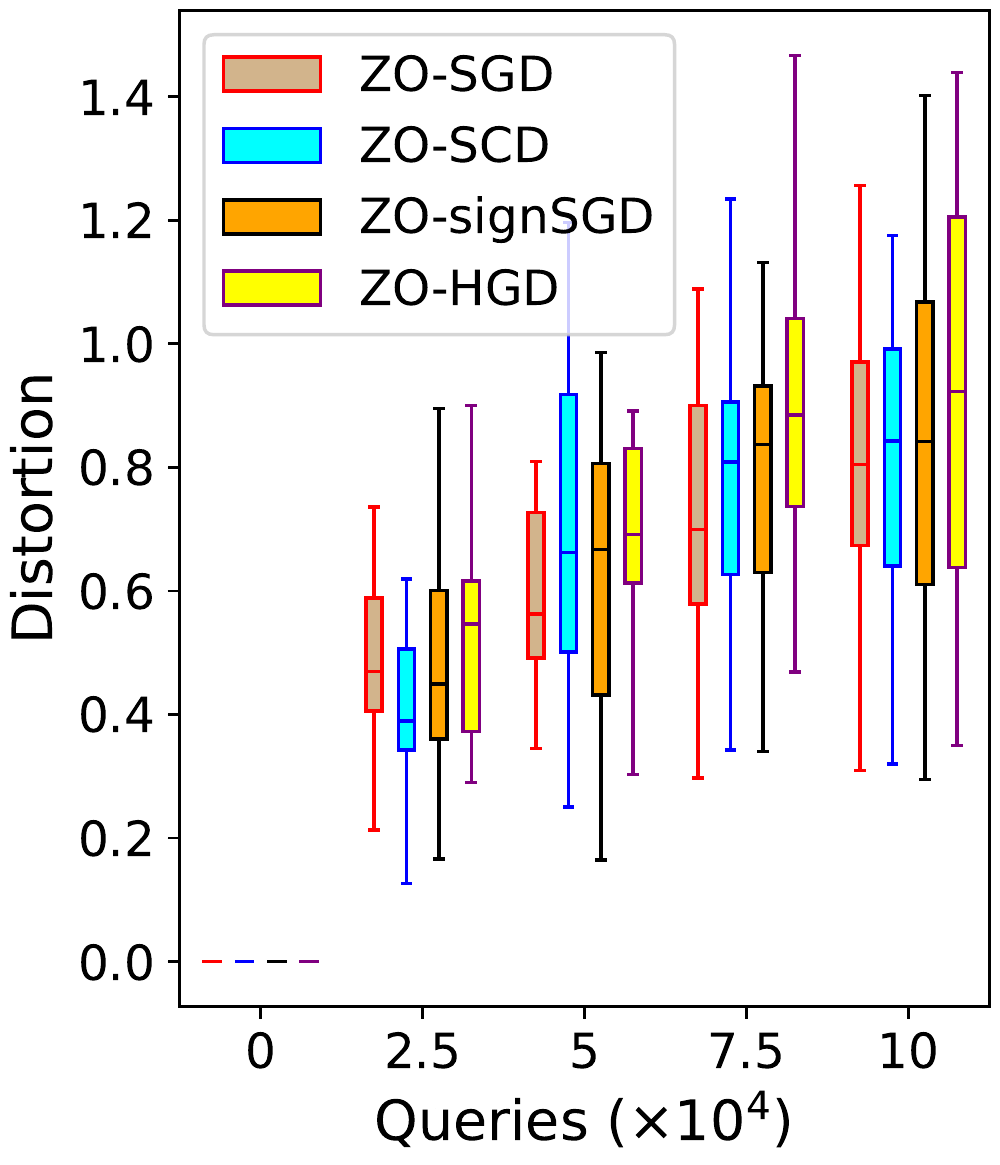} \\
\footnotesize{(a)}  & \footnotesize{(b)}  
\end{tabular}
\end{center}
\caption{\footnotesize{Comparing ZO-HGD with ZO-SGD, ZO-SCD and ZO-signSGD for the task of generating universal adversarial perturbations in  (a)  attack success rate, and (b) $\ell_2$ distortion   of obtained universal perturbations versus number of queries. Here the box plot summarizes the results of $50$ random trials to solve problem    \eqref{eq: attack_general}.
}}
  \label{fig: loss_dist_universal_attack}
\end{figure}

\begin{table}[htb]
\centering 
  \begin{tabular}
      {cccccc}
      \toprule
       Original label & horse & car & bird & flight & frog \\ 
       \hline
      	\begin{tabular}[c]{@{}c@{}}ZO-SGD\\ $\ell_2$: 0.67\end{tabular} &
        \parbox[c]{2.2em}{\includegraphics[width=0.38in]{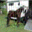}}&
        \parbox[c]{2.2em}{\includegraphics[width=0.38in]{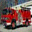}}&
        \parbox[c]{2.2em}{\includegraphics[width=0.38in]{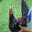}}&
        \parbox[c]{2.2em}{\includegraphics[width=0.38in]{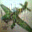}}&
        \parbox[c]{2.2em}{\includegraphics[width=0.38in]{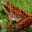}} \\
         	Query \# 	& 37 & 255 & 25 & 176 & 635    	\\
        Predicted label & dog & ship & dog & car & car\\        
      \hline 
      	\begin{tabular}[c]{@{}c@{}}ZO-SCD\\ $\ell_2$: 0.71\end{tabular} &
        \parbox[c]{2.2em}{\includegraphics[width=0.38in]{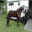}}&
        \parbox[c]{2.2em}{\includegraphics[width=0.38in]{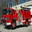}}&
        \parbox[c]{2.2em}{\includegraphics[width=0.38in]{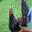}}&
        \parbox[c]{2.2em}{\includegraphics[width=0.38in]{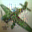}}&
        \parbox[c]{2.2em}{\includegraphics[width=0.38in]{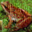}} \\
         	Query \# 	& 35 & 239 & 21 & 166 & 610    	\\
        Predicted label & dog & ship & dog & car & car\\        
      \hline 
      	\begin{tabular}[c]{@{}c@{}}ZO-signSGD\\ $\ell_2$: 0.69\end{tabular} &
        \parbox[c]{2.2em}{\includegraphics[width=0.38in]{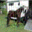}}&
        \parbox[c]{2.2em}{\includegraphics[width=0.38in]{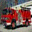}}&
        \parbox[c]{2.2em}{\includegraphics[width=0.38in]{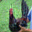}}&
        \parbox[c]{2.2em}{\includegraphics[width=0.38in]{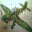}}&
        \parbox[c]{2.2em}{\includegraphics[width=0.38in]{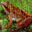}} \\
         	Query \# 	& 39 & 285 & 29 & 188 & 675    	\\
        Predicted label & dog & ship & dog & car & car\\        
      \hline 
      	\begin{tabular}[c]{@{}c@{}}ZO-HGD\\ $\ell_2$: 0.77\end{tabular} &
        \parbox[c]{2.2em}{\includegraphics[width=0.38in]{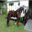}}&
        \parbox[c]{2.2em}{\includegraphics[width=0.38in]{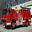}}&
        \parbox[c]{2.2em}{\includegraphics[width=0.38in]{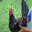}}&
        \parbox[c]{2.2em}{\includegraphics[width=0.38in]{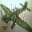}}&
        \parbox[c]{2.2em}{\includegraphics[width=0.38in]{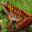}} \\
         	Query \# 	& 30 & 210 & 19 & 152 & 573    	\\
        Predicted label & dog & ship & dog & car & car\\        
\bottomrule
      
  \end{tabular}
\caption{Randomly selected 5 over 10 images in one trial from our universal black-box attack  generation using ZO-SGD, ZO-SCD, ZO-signSGD and ZO-HGD (our proposal). The results are summarized in terms of  the number of queries required to achieve the first successful attack, the resulting adversarial examples, predicted label, and final $\ell_2$ distortion.
} \label{table:cifar-image}
\end{table}

In experiments, we consider  a DNN model with $5$ convolutional layers and $2$ fully connected layers, and train it  over the CIFAR-10 dataset for image classification \cite{krizhevsky2009learning}.
We focus on the   scenario to generate universal adversarial perturbations. 
Specifically, we  
conduct $50$ random trials to solve problem \eqref{eq: attack_general}, each of which randomly selects $M = 10$ images  from CIFAR-10 testing data and sets $\lambda = 10$.  
We compare our proposed ZO-HGD algorithm with three baselines,  ZO-SGD \cite{Ghadimi_Siam_2013_SGD}, ZO-SCD \cite{lian2016comprehensive} and ZO-signSGD~\cite{liu2019signsgd}. 
The rationale behind comparing with these baselines is that ZO-SGD has demonstrated a superior performance in query efficiency over ZO optimization methods that use extra variance reduction techniques \cite{liu2018zeroth}, and ZO-SCD and ZO-signSGD are used as the backbones of many  black-box   attack generation algorithms via ZO optimization \cite{chen2017zoo,ilyas2018black}.
In ZO-HGD (Algorithm\,\ref{Algo_zo_hgd}), we set $\nC = 50$ as the  coordinate selection budget  in CGE and $n_{\mathrm{r}} = 50$ when constructing  RGE, and choose the combination coefficient to weight RGE and CGE in HGE as $\alpha_t = t/T$. In all of the methods, we set the maximum number of iterations $T = 1000$, the same query budget $(\nC + n_{\mathrm{r}})= 100$ per iteration, and    a constant smoothing parameter $0.001$. We pick the best learning rate for each method by greedily searching over the interval $[10^{-4}, 10^{-1}]$.



In Figure\,\ref{fig: obj_value_universal_attack}, we present  the averaged objective value of problem \eqref{eq: attack_general} over $50$ random trials versus the number of function queries.  And in Figure\,\ref{fig: loss_dist_universal_attack}, we show the attack success rate (ASR) of $50 \times 10$ perturbed images, and the associated $\ell_2$-norm distortion for attack generation.  
We compare our proposed ZO-HGD with ZO-SGD, ZO-SCD  and ZO-signSGD.
As we can see, the convergence speed of ZO-HGD is faster than ZO-SGD and ZO-SCD. This is supported by Figure\,\ref{fig: obj_value_universal_attack}:   
A smaller objective value is achieved using less number of function queries than the other algorithms. By dissecting the objective value of \eqref{eq: attack_general},   Figure\,\ref{fig: loss_dist_universal_attack}-(a) shows that  ZO-HGD yields a   significant ASR improvement over other algorithms even at the early iterations (e.g., at the use of   $2.5\sim 5 \times 10^4$ queries). 
Moreover, we observe from Figure\,\ref{fig: loss_dist_universal_attack}-(b) that all the considered ZO algorithms result in comparable $\ell_2$ distortion strength in general. 
This implies that
 the ASR improvement introduced by ZO-HGD is not at a significant cost of  increasing distortion norm.
 To further support this point,
Table\,\ref{table:cifar-image} presents the concrete  adversarial examples and their metrics obtained from different ZO optimization methods.    As we can see, ZO-HGD requires the least number of function queries to achieve the first successful attack of each image, and keeps the strength of converged perturbations comparable across methods.
 

\section{Conclusion}
\label{sec_conclusion}
In this paper, we proposed a novel hybrid gradient estimator (HGE) for black-box stochastic optimization, combining the advantages of the query-efficient random gradient estimator, and the more accurate coordinate-wise gradient estimator. 
Using importance sampling based coordinate selection, we further improved the variance of HGE. 
Building on top of this estimator, we proposed a ZO-hybrid gradient descent (ZO-HGD) algorithm, which generalized ZO-SGD and ZO-SCD. 
We conducted theoretical analysis of the method for smooth functions, which are non-convex, convex and strongly convex, and rigorously demonstrated the efficiency in terms of both iteration complexity and function query cost.
The theoretical findings have been corroborated by a real-world application to generate adversarial examples from a black-box deep neural network.
The generalization of this work to nonsmooth problems is one of the future directions we plan to pursue.



\bibliographystyle{IEEEtran}
\bibliography{abrv,References}

\onecolumn
\setcounter{section}{0}
\setcounter{figure}{0}
\setcounter{equation}{0}
\makeatletter 
\renewcommand{\thefigure}{A\@arabic\c@figure}
\makeatother
\setcounter{table}{0}
\renewcommand{\thetable}{A\arabic{table}}
\renewcommand{\theequation}{S\arabic{equation}}
\setcounter{algorithm}{0}
\renewcommand{\thealgorithm}{A\arabic{algorithm}}

\begin{appendices}

\section{Mathematical Notations}
\label{app_notation}

\begin{table}[h!]
\centering
\label{table_notation}
\begin{tabular}{ |c|l| }
    \hline
    \multicolumn{2}{|c|}{General notations} \\
    \hline
    $f$ & Objective function to be minimized, defined in \eqref{eq_problem} \\
        $\mathrm{dom} f$ & Domain of the function $f$ \\
    $d$ & Dimension of the problem \\
    $F(\cdot; \xi)$ & Component functions of the objective function $f$ at a stochastic sample $\xi$\\
    $\Xi$ & Distribution of $\xi$
    \\
    $\mathcal{X}, |\mathcal{X}|$ & Calligraphic font is used to denote sets, $|\mathcal{X}|$ denotes the cardinality of set $\mathcal{X}$ \\
    $\mathbb{N}_+$ & The set of positive integers \\
    $[n], n \in \mathbb{N}_+$ & The set $\{ 1, 2, \hdots, n \}$ \\
    $\| \cdot \|$ & Euclidean norm (unless stated otherwise) \\
    $\left\langle \cdot, \cdot \right\rangle$ & Inner product \\
    \hline
    \multicolumn{2}{|c|}{Assumptions} \\
    \hline
    $L$ & Parameter of gradient Lipschitz continuity  \\
    $\zeta^2$ & Coordinate-wise variance bound: $\mathbb{E} [ ( \nabla F(\mathbf{x}; \xi) - \nabla f(\mathbf{x}) )_i^2 ] \leq \zeta^2, \forall \ i$ \\
    $\sigma^2$ & $= d \zeta^2$ \\
    $R$ & Upper bound on the diameter of the set $\mathrm{dom} f$, i.e., $\| \mathbf{y} - \mathbf{x} \| \leq R$, $\forall \ \mathbf{x}, \mathbf{y} \in \mathrm{ dom } f$ \\
    $G$ & Upper bound on the gradient norm: $\left\| \nabla f(\mathbf{x}) \right\| \leq G, \forall \mathbf{x} \in \mathrm{ dom } f$ \\
    $\bar{\sigma}$ & Strong convexity parameter \\
    \hline
    \multicolumn{2}{|c|}{RGE $\Hat{\nabla}_{\mathrm{RGE}} F (\mathbf x; \br)$; see \eqref{eq_RGE}, \eqref{eq: GE_stoc_batch}} \\
    \hline
    $\{ \mbf u_i \}$ & Random directions, sampled from the uniform distribution on a unit radius sphere $U_0$, centered at $\mbf 0$ \\
    $\nR$ & Number of  random directions, per stochastic sample $\xi$ \\
    $\muR$ & Smoothing parameter \\
    $\br$ & Mini-batch of stochastic samples with cardinality $|\br|$\\
    \hline
    \multicolumn{2}{|c|}{CGE $\Hat{\nabla}_{\mathrm{CGE}} F_{\mathcal{I}} (\mathbf{x}; \bc, \mathbf{p})$; see \eqref{eq_CGE_I}, \eqref{eq: GE_stoc_batch}} \\
    \hline
    $\muCi$ & Smoothing parameter for the $i$-th component of CGE; see \eqref{eq_CGE_1} \\
    $\muC$ & Common smoothing parameter across all the coordinates, i.e., $\muCi = \muC$ for all $i \in [d]$ \\
    $\{ \mbf e_i \}_{i=1}^d$ & Canonical basis vectors in $\mathbb{R}^d$ ($\mbf e_i$ is a vector of all $0$'s, except the $i$-th entry) \\
    $\mathcal{I}$ & Coordinate set used in computing CGE \eqref{eq_CGE_I} \\
    $\bc$ & Mini-batch of stochastic samples with cardinality $|\bc|$  \\
    $\mbf p$ & Vector of coordinate selection probabilities \eqref{eq_CGE_nonuniform} \\
    $p_i$ & $\mathrm{Pr} (i \in \mathcal{I})$; $i$-th element of $\mbf p$ \\
    \hline
    \multicolumn{2}{|c|}{HGE \eqref{eq_HGE}: $\Hat{\nabla}_{\mathrm{HGE}} F(\mathbf x; \br, \bc, \mathcal I) =  \alpha \Hat{\nabla}_{\mathrm{RGE}} F (\mathbf x; \br) + (1-\alpha) 
    \Hat{\nabla}_{\mathrm{CGE}} F_{\mathcal{I}} (\mathbf{x}; \bc, \mathbf{p})$} \\
    \hline
    $\alpha$ & Convex combination coefficient in HGE \\
    $\alpha^*$ & Optimal value of $\alpha$, obtained by minimizing the bound on the variance of HGE (Proposition \ref{prop_bd_var_RGE_CGE}) \\
    $\nC$ & Coordinate selection budget of CGE \\
    \hline
    \multicolumn{2}{|c|}{Algorithm \ref{Algo_zo_hgd}: HGD - $\mathbf{x}_{t+1} = \mathbf{x}_{t} - \eta_t \left( \alpha_t \nabla_{\mathrm{r},t} + (1 - \alpha_t) \nabla_{\mathrm{c},t} \right)$} \\
    \hline
    $\mathbf{x}_t$ & Iterate of the algorithm at the $t$-th time step \\
    $T$ & Total number of iterations  \\
    $\eta_t$ & Step-size at time $t$\\
    $n_{\mathrm{c,t}}$ & Coordinate selection budget of CGE at time $t$ \\
    $\alpha_t, \brt, \mathcal{I}_t, \bct, \mathbf{p}_t$ & The time-dependent versions of the corresponding quantities ($\alpha, \br, \mathcal{I}, \bc, \mathbf{p}$) defined above \\
    $p_{t,i}$ & $i$-th element of $\mathbf{p}_t$ \\
    $\nabla_{\mathrm{r},t}$ & $\Hat{\nabla}_{\mathrm{RGE}} F (\mathbf{x}_t; \brt)$ \\
    $\nabla_{\mathrm{c},t}$ & $\Hat{\nabla}_{\mathrm{CGE}} F_{\mathcal{I}_t} (\mathbf{x}_t; \bct, \mathbf{p}_t)$ \\
    $Y_t$ & $= \{ \{ \mathbf{u}_{t,i} \}_i, \mathcal{I}_t, \bct, \brt \}$, the randomness at step $t$ \\
    $\mathcal{Y}_t$ & \begin{tabular}[l]{@{}l@{}} $\sigma$-algebra generated by $= \{ Y_0, Y_1, \hdots, Y_{t-1} \}$. $\mathcal{Y}_t$ denotes the entire randomness of the algorithm \\ upto time $t-1$ \end{tabular} \\
    \hline
    \multicolumn{2}{|c|}{Quantities used in the analysis of Algorithm \ref{Algo_zo_hgd}} \\
    \hline
    $\dnr$ & $1 + d/\nR$ \\
    $\bar{P}_T$ & $\frac{1}{T} \sum_{t=0}^{T-1} \frac{1}{d} \sum_{i=1}^d \frac{1}{p_{t,i}}$ \\
    $c_t$ & Lower bound such that $p_{t,i} \geq c_t > 0$ for all $i \in [d]$, at all times $t$ \\
    $\frac{1}{\bar{c}_T}$ & $\frac{1}{T} \sum_{t=0}^{T-1} \frac{1}{c_t}$ \\
    $q$ & $= \nR + \nC$, such that the total per-iteration function query cost is $O(q)$ \\
    \hline
\end{tabular}
\end{table}

\newpage
\section{Hybrid Gradient Estimator}
\subsection{Proof of Proposition \ref{prop_prob_opt}}
\label{proof_prop_prob_opt}
The proof is motivated by the analysis of gradient sparsification proposed in \cite{wangni2018gradient}.
We denote by $\mathbf{g} = [ g_{1}, \hdots, g_{d}]^T \triangleq \Hat{\nabla}_{\mathrm{RGE}} F (\mathbf x; \br)$, the RGE to be employed as a probe estimate.  
We also use the vector $[ g_{(1)}, \hdots, g_{(d)}]^T$ to denote a vector with the same entries as $\mathbf{g}$, but sorted in the order of decreasing magnitude. We denote the vector of probabilities $\mathbf{p} = [ p_1, \hdots, p_d]^T$.

Given an upper bound on the expected sparsity $\sum_{i=1}^d p_i$, we compute the set of probabilities which minimize the variance of the sparse estimator $Q(\mathbf{g})$, by solving the following problem. 
\begin{align}
	\min \sum_{i=1}^d \frac{g_{i}^2}{p_i} \qquad \text{s.t.} \qquad \sum_{i=1}^d p_i \leq \nC, 0 < p_i \leq 1, \forall i. \label{app_eq_prob_opt}
\end{align}

\begin{proof}
Introducing Lagrange multipliers $\lambda, \{ \mu_i \}_{i=1}^d$ in \eqref{app_eq_prob_opt}, the solution is given by solving the following problem:
\begin{align}
	\min_p \max_{\lambda \geq 0} \max_{\mu \succeq \mathbf{0}} L(p, \lambda, \mu) = \sum_{i=1}^d \frac{g_{i}^2}{p_i} + \lambda^2 \left( \sum_{i=1}^d p_i - \nC \right) + \sum_{i=1}^d \mu_i (p_i - 1). \label{proof_prob_opt_1}
\end{align}
The KKT conditions for \eqref{proof_prob_opt_1} are given by,
\begin{align*}
	-\frac{g_{i}^2}{p_i^2} + \lambda^2 + \mu_i = 0, \forall \ i \qquad & \text{(Stationarity)} \nonumber \\
	\lambda^2 \left( \sum_{i=1}^d p_i - \nC \right) = 0, \qquad & \text{(Complementary slackness 1)} \nonumber \\
	\mu_i (p_i - 1) = 0, \forall \ i, \qquad & \text{(Complementary slackness 2)}. \nonumber
\end{align*}
Combining the stationarity condition with the fact that for $p_i < 1$, $\mu_i = 0$, we get
\begin{align}
	p_i = 
	\begin{cases}
		1, & \mu_i \neq 0 \\
		\frac{| g_{i} |}{\lambda}, & \mu_i = 0. \label{proof_prob_opt_2}
	\end{cases}.
\end{align}
It follows from \eqref{proof_prob_opt_2} that if $| g_i | \geq | g_j |$, then $p_i \geq p_j$. Therefore, there is a dominating set of coordinates, say $S$, such that $p_j = 1, \forall j \in S$, and $| g_j |, j \in S$ are the largest absolute magnitudes in the vector $\mathbf{g}$. Suppose the set $S$ has size $k (0 \leq k \leq d)$. Consequently, we get,
\begin{align}
	p_{(i)} = 
	\begin{cases}
		1, & i \leq k \\
		\frac{| g_{(i)} |}{\lambda}, & i > k. \label{proof_prob_opt_3}
	\end{cases}.
\end{align}
Clearly, $\lambda > 0$, hence $\sum_{i=1}^d p_i = \nC$. Then, from \eqref{proof_prob_opt_3},
\begin{align*}
	k + \frac{1}{\lambda} \sum_{i = k+1}^d | g_{(i)} | = \nC \qquad \Leftrightarrow \qquad \lambda = \frac{\sum_{i = k+1}^d | g_{(i)} |}{\nC-k}.
\end{align*}
Consequently, for $i > k$,
\begin{align*}
	p_i = \frac{| g_{(i)} | (\nC-k)}{\sum_{i = k+1}^d | g_{(i)} |}
\end{align*}
Hence, the probabilities are calculated using the following steps:
\begin{enumerate}
\item Find the smallest $k$ such that
	\begin{align*}
		| g_{(k+1)} | \left( \nC-k \right) \leq \sum_{i=k+1}^d | g_{(i)} |,
	\end{align*}
	is true. Denote by $S_k$, the set of coordinates with the top $k$ largest magnitudes of $| g_i |$.
\item Set the probability vector $\mathbf{p}$ by
	\begin{align*}
		p_{i} = 
		\begin{cases}
			1, & \text{ if } i \in S_k \\
			\frac{| g_{i} | (\nC-k)}{\sum_{i = k+1}^d | g_{i} |}, & \text{ if } i \notin S_k.
		\end{cases}
	\end{align*}
\end{enumerate}
\end{proof}

\subsection{Proof of Proposition \ref{prop_bd_var_RGE_CGE}}
\label{proof_prop_bd_var_RGE_CGE}
\subsubsection{Proof of \eqref{eq_bd_var_RGE}}
\begin{proof}
The variance of RGE is given by
\begin{align}
    \mathbb{E} \left\| \Hat{\nabla}_{\mathrm{RGE}} F \left( \mathbf x; \br \right) - \nabla f \left( \mathbf{x} \right) \right\|^2 & = \mathbb{E} \left\| \Hat{\nabla}_{\mathrm{RGE}} F \left(\mathbf x; \br \right) - \nabla f_{\muR} (\mathbf{x}) + \nabla f_{\muR} (\mathbf{x}) - \nabla f (\mathbf{x}) \right\|^2 \nonumber \\
    & \overset{(a)}{=} \mathbb{E} \left\| \Hat{\nabla}_{\mathrm{RGE}} F \left(\mathbf x; \br \right) - \nabla f_{\muR} (\mathbf{x}) \right\|^2 + \left\| \nabla f_{\muR} (\mathbf{x}) - \nabla f (\mathbf{x}) \right\|^2 \label{proof_eq_bd_var_RGE_1}
\end{align}
where $(a)$ follows since $\mathbb{E} \Hat{\nabla}_{\mathrm{RGE}} F (\mathbf x; \br) = \nabla f_{\muR} (\mathbf x)$ (see the discussion following \eqref{eq_RGE}). Next, we upper bound the first term in \eqref{proof_eq_bd_var_RGE_1}.
\begin{align}
    & \mathbb{E} \left\| \Hat{\nabla}_{\mathrm{RGE}} F \left(\mathbf x; \br \right) - \nabla f_{\muR} (\mathbf{x}) \right\|^2 = \mathbb{E}_{\{ \mathbf{u} \}, \br} \left\| \Hat{\nabla}_{\mathrm{RGE}} F \left(\mathbf x; \br \right) - \nabla f_{\muR} (\mathbf{x}) \right\|^2 \nonumber \\
    & \qquad \overset{(b)}{=} \mathbb{E}_{\{ \mathbf{u} \}, \br} \left\| \frac{1}{|\br|} \sum_{\xi \in \br} \frac{1}{\nR} \sum_{i=1}^{\nR} \frac{d \left [ F(\mathbf{x} + \muR \mathbf{u}_{i,\xi}; \xi) - F(\mathbf{x}; \xi) \right ]}{\muR}  \mathbf{u}_{i,\xi} - \nabla f_{\muR} (\mathbf{x}) \right\|^2 \nonumber \\
    & \qquad = \frac{1}{|\br|^2 \nR^2} \sum_{\xi \in \br} \mathbb{E}_{\{ \mathbf{u} \}, \xi} \left\| \sum_{i=1}^{\nR} \left( \frac{d \left [ F(\mathbf{x} + \muR \mathbf{u}_{i,\xi}; \xi) - F(\mathbf{x}; \xi) \right ]}{\muR}  \mathbf{u}_{i,\xi} - \nabla f_{\muR} (\mathbf{x}) \right) \right\|^2 \nonumber \\
    & \qquad \qquad + \frac{1}{|\br|^2 \nR^2} \sum_{\xi \neq \chi} \left\langle \mathbb{E}_{\xi} \sum_{i=1}^{\nR} \mathbb{E}_{\mathbf{u}_{i,\xi}} \left( \frac{d \left [ F(\mathbf{x} + \muR \mathbf{u}_{i,\xi}; \xi) - F(\mathbf{x}; \xi) \right ]}{\muR}  \mathbf{u}_{i,\xi} - \nabla f_{\muR} (\mathbf{x}) \right), \right. \nonumber \\
    & \qquad \qquad \qquad \qquad \qquad \qquad \qquad \left. \mathbb{E}_{\chi} \sum_{i=1}^{\nR} \mathbb{E}_{\mathbf{u}_{i,\chi}} \left( \frac{d \left [ F(\mathbf{x} + \muR \mathbf{u}_{i,\chi}; \chi) - F(\mathbf{x}; \chi) \right ]}{\muR}  \mathbf{u}_{i,\chi} - \nabla f_{\muR} (\mathbf{x}) \right) \right\rangle \nonumber \\
    & \qquad \overset{(c)}{=} \frac{1}{|\br|^2 \nR^2} \sum_{\xi \in \br} \sum_{i=1}^{\nR}  \mathbb{E}_{ \mathbf{u}_{i,\xi}, \xi} \left\| \frac{d \left [ F(\mathbf{x} + \muR \mathbf{u}_{i,\xi}; \xi) - F(\mathbf{x}; \xi) \right ]}{\muR}  \mathbf{u}_{i,\xi} - \nabla f_{\muR} (\mathbf{x}) \right\|^2 \nonumber \\
    & \qquad \qquad + \frac{1}{|\br|^2 \nR^2} \sum_{\xi \in \br} \sum_{i \neq j} \mathbb{E}_{\xi} \left\langle \mathbb{E}_{\mathbf{u}_{i,\xi}} \left( \frac{d \left [ F(\mathbf{x} + \muR \mathbf{u}_{i,\xi}; \xi) - F(\mathbf{x}; \xi) \right ]}{\muR}  \mathbf{u}_{i,\xi} - \nabla f_{\muR} (\mathbf{x}) \right), \right. \nonumber \\
    & \qquad \qquad \qquad \qquad \qquad \qquad \qquad \qquad \qquad \left. \mathbb{E}_{\mathbf{u}_{j,\xi}} \left( \frac{d \left [ F(\mathbf{x} + \muR \mathbf{u}_{j,\xi}; \xi) - F(\mathbf{x}; \xi) \right ]}{\muR}  \mathbf{u}_{j,\xi} - \nabla f_{\muR} (\mathbf{x}) \right) \right\rangle \nonumber \\
    & \qquad \overset{(d)}{=} \frac{1}{|\br| \nR} \mathbb{E}_{\xi_1, \mathbf{u}_{1,\xi_1}} \left\| \frac{d \left [ F(\mathbf{x} + \muR \mathbf{u}_{1,\xi_1}; \xi_1) - F(\mathbf{x}; \xi_1) \right ]}{\muR}  \mathbf{u}_{1,\xi_1} - \nabla f_{\muR} (\mathbf{x}) \right\|^2 \nonumber \\
    & \qquad \qquad + \frac{(\nR - 1)}{|\br| \nR} \mathbb{E}_{\xi_1} \left\| \nabla F_{\muR} (\mathbf{x}, \xi_1) - \nabla f_{\muR} (\mathbf{x}) \right\|^2 \qquad \text{ (for some } \xi_1 \in \br), \label{proof_eq_bd_var_RGE_2}
\end{align}
where $(b)$ follows from the definition of RGE \eqref{eq_RGE}, \eqref{eq: GE_stoc_batch}. $\{ \mathbf{u}_{i,\xi} \}$ denotes the set of random directions associated with sample $\xi$. $(c)$ follows since the samples of $\br$ are picked independently of each other, and the random directions associated with sample $\xi$, $\{ \mathbf{u}_{i,\xi} \}$ are also independent of the random directions corresponding to another sample $\chi$. Also,
\begin{align*}
    \mathbb{E}_{\xi} \mathbb{E}_{\mathbf{u}_{i,\xi}} \left[ \frac{d \left [ F(\mathbf{x} + \muR \mathbf{u}_{i,\xi}; \xi) - F(\mathbf{x}; \xi) \right ]}{\muR}  \mathbf{u}_{i,\xi} \right] = \nabla f_{\muR} (\mathbf{x}).
\end{align*}
Hence,
\begin{align*}
    & \left\langle \mathbb{E}_{\xi} \sum_{i=1}^{\nR} \mathbb{E}_{\mathbf{u}_{i,\xi}} \left( \frac{d \left [ F(\mathbf{x} + \muR \mathbf{u}_{i,\xi}; \xi) - F(\mathbf{x}; \xi) \right ]}{\muR}  \mathbf{u}_{i,\xi} - \nabla f_{\muR} (\mathbf{x}) \right), \right. \nonumber \\
    & \qquad \qquad \qquad \left. \mathbb{E}_{\chi} \sum_{i=1}^{\nR} \mathbb{E}_{\mathbf{u}_{i,\chi}} \left( \frac{d \left [ F(\mathbf{x} + \muR \mathbf{u}_{i,\chi}; \chi) - F(\mathbf{x}; \chi) \right ]}{\muR}  \mathbf{u}_{i,\chi} - \nabla f_{\muR} (\mathbf{x}) \right) \right\rangle = 0.
\end{align*}
Further, in \eqref{proof_eq_bd_var_RGE_2}, $(d)$ follows since
\begin{align*}
    \mathbb{E}_{\mathbf{u}} \frac{d \left [ F(\mathbf{x} + \muR \mathbf{u}; \xi) - F(\mathbf{x}; \xi) \right ]}{\muR}  \mathbf{u} = \nabla f_{\muR} (\mathbf{x}; \xi).
\end{align*}
Also, both the samples $\xi \in \br$ and the random directions $\{ \mathbf{u}_{i,\xi} \}$ are picked independently and uniformly. We denote by $\xi_1$ and $\mathbf{u}_{i,\xi}$, a representative sample of both sets respectively.
Next, we upper bound the two terms in \eqref{proof_eq_bd_var_RGE_2}. From Assumption \ref{assum_var_bound} and \cite[Lemma~4.2]{gao18information}, we obtain
\begin{align}
    \mathbb{E}_{\xi, \mathbf{u}} \left\| \frac{d \left [ F(\mathbf{x} + \muR \mathbf{u}; \xi) - F(\mathbf{x}; \xi) \right ]}{\muR}  \mathbf{u} - \nabla f_{\muR} (\mathbf{x}) \right\|^2 \leq 2d \left[ \left\| \nabla f (\mathbf{x}) \right\|^2 + \sigma^2 \right] + \frac{\muR^2 L^2 d^2}{2}. \label{proof_eq_bd_var_RGE_3}
\end{align}
Also, the second term in \eqref{proof_eq_bd_var_RGE_2} can be upper bounded as follows:
\begin{align}
    \mathbb{E}_{\xi} \left\| \nabla F_{\muR} (\mathbf{x}, \xi) - \nabla f_{\muR} (\mathbf{x}) \right\|^2 & \leq \mathbb{E}_{\xi} \left\| \nabla F_{\muR} (\mathbf{x}, \xi) \right\|^2 \nonumber \\
    & \leq 2 \mathbb{E}_{\xi} \left\| \nabla F (\mathbf{x}, \xi) \right\|^2 + \frac{\muR^2 L^2 d^2}{2} \qquad \qquad \qquad \tag*{\text{from \cite[Lemma~1]{liu2018zeroth}}}  \nonumber \\
    & \leq 2 \left[ \mathbb{E}_{\xi} \left\| \nabla F (\mathbf{x}, \xi) - \nabla f(\mbf x) \right\|^2 + \left\| \nabla f(\mbf x) \right\|^2 \right] + \frac{\muR^2 L^2 d^2}{2} \nonumber \\
    & \leq 2 \left[ \left\| \nabla f (\mathbf{x}) \right\|^2 + \sigma^2 \right] + \frac{\muR^2 L^2 d^2}{2}. \label{proof_eq_bd_var_RGE_4}
\end{align}
Substituting \eqref{proof_eq_bd_var_RGE_3}, \eqref{proof_eq_bd_var_RGE_4} in \eqref{proof_eq_bd_var_RGE_2}, we get
\begin{align}
    & \mathbb{E} \left\| \Hat{\nabla}_{\mathrm{RGE}} F \left(\mathbf x; \br \right) - \nabla f_{\muR} (\mathbf{x}) \right\|^2 \leq \frac{2}{|\br|} \left( 1 + \frac{d}{\nR} \right) \left[ \left\| \nabla f (\mathbf{x}) \right\|^2 + \sigma^2 \right] + \left( 1 + \frac{1}{\nR} \right) \frac{\muR^2 L^2 d^2}{2 |\br|}. \label{proof_eq_bd_var_RGE_5}
\end{align}
Substituting \eqref{proof_eq_bd_var_RGE_5} in \eqref{proof_eq_bd_var_RGE_1}, and using $\left\| \nabla f_{\muR} (\mathbf{x}) - \nabla f (\mathbf{x}) \right\| \leq \frac{\muR d L}{2}$ from \cite[Lemma~1]{liu2018zeroth}, we get
\begin{align}
    \mathbb{E} \left\| \Hat{\nabla}_{\mathrm{RGE}} F \left( \mathbf x; \br \right) - \nabla f \left( \mathbf{x} \right) \right\|^2 & \leq \frac{2}{|\br|} \left( 1 + \frac{d}{\nR} \right) \left[ \left\| \nabla f (\mathbf{x}) \right\|^2 + \sigma^2 \right] + \left( 1 + \frac{2}{|\br|} + \frac{2}{\nR |\br|} \right) \frac{\muR^2 L^2 d^2}{4}.
\end{align}
\end{proof}

\subsubsection{Proof of \eqref{eq_bd_var_CGE}}
\begin{proof}
The variance of CGE is bounded as
\begin{align}
    \mathbb{E} \left\| \Hat{\nabla}_{\mathrm{CGE}} F_{\mathcal{I}} (\mathbf{x}; \bc, \mathbf{p}) - \nabla f (\mathbf{x}) \right\|^2 &= \mathbb{E} \left\| \Hat{\nabla}_{\mathrm{CGE}} F_{\mathcal{I}} (\mathbf{x}; \bc, \mathbf{p}) - \Hat{\nabla}_{\mathrm{CGE}} f(\mathbf{x}) + \Hat{\nabla}_{\mathrm{CGE}} f(\mathbf{x}) - \nabla f (\mathbf{x}) \right\|^2 \nonumber \\
    & \overset{(e)}{=} \mathbb{E} \left\| \Hat{\nabla}_{\mathrm{CGE}} F_{\mathcal{I}} (\mathbf{x}; \bc, \mathbf{p}) - \Hat{\nabla}_{\mathrm{CGE}} f(\mathbf{x}) \right\|^2 + \left\| \Hat{\nabla}_{\mathrm{CGE}} f(\mathbf{x}) - \nabla f (\mathbf{x}) \right\|^2, \label{proof_prop_eq_bd_var_CGE_1}
\end{align}
where $\Hat{\nabla}_{\mathrm{CGE}} f(\mathbf{x})$ is the full coordinate CGE, and $(e)$ follows since $\mathbb{E} [ \Hat{\nabla}_{\mathrm{CGE}} F_{\mathcal{I}} (\mathbf{x}; \bc, \mathbf{p})] = \Hat{\nabla}_{\mathrm{CGE}} f(\mathbf{x})$. Now, we bound the first term in \eqref{proof_prop_eq_bd_var_CGE_1}.
\begin{align}
    & \mathbb{E} \left\| \Hat{\nabla}_{\mathrm{CGE}} F_{\mathcal{I}} (\mathbf{x}; \bc, \mathbf{p}) - \Hat{\nabla}_{\mathrm{CGE}} f(\mathbf{x}) \right\|^2 \overset{(f)}{=} \mathbb{E} \left\| \sum_{i=1}^d \frac{I(i \in \mathcal{I})}{p_i} \Hat{\nabla}_{\mathrm{CGE}} F_i (\mathbf{x}; \bc) - \sum_{i=1}^d \Hat{\nabla}_{\mathrm{CGE}} f_i(\mathbf{x}) \right\|^2 \nonumber \\
    & \qquad \overset{(g)}{=} \sum_{i=1}^d \mathbb{E} \left\| \frac{I(i \in \mathcal{I})}{p_i} \Hat{\nabla}_{\mathrm{CGE}} F_i (\mathbf{x}; \bc) - \Hat{\nabla}_{\mathrm{CGE}} f_i(\mathbf{x}) \right\|^2 \nonumber \\
    & \qquad = \sum_{i=1}^d \frac{1}{p_i^2} \mathbb{E} \left\| I(i \in \mathcal{I}) \Hat{\nabla}_{\mathrm{CGE}} F_i (\mathbf{x}; \bc) - I(i \in \mathcal{I}) \Hat{\nabla}_{\mathrm{CGE}} f_i(\mathbf{x}) + I(i \in \mathcal{I}) \Hat{\nabla}_{\mathrm{CGE}} f_i(\mathbf{x}) - p_i \Hat{\nabla}_{\mathrm{CGE}} f_i(\mathbf{x}) \right\|^2 \nonumber \\
    & \qquad \overset{(j)}{=} \sum_{i=1}^d \frac{1}{p_i^2} \left[ \mathbb{E} \left\| I(i \in \mathcal{I}) \Hat{\nabla}_{\mathrm{CGE}} F_i (\mathbf{x}; \bc) - I(i \in \mathcal{I}) \Hat{\nabla}_{\mathrm{CGE}} f_i(\mathbf{x}) \right\|^2 + \mathbb{E} \left\| I(i \in \mathcal{I}) \Hat{\nabla}_{\mathrm{CGE}} f_i(\mathbf{x}) - p_i \Hat{\nabla}_{\mathrm{CGE}} f_i(\mathbf{x}) \right\|^2 \right] \nonumber \\
    & \qquad = \sum_{i=1}^d \frac{1}{p_i^2} \left[ \mathbb{E}_{\mathcal{I}} I^2(i \in \mathcal{I}) \mathbb{E}_{\bc} \left\| \Hat{\nabla}_{\mathrm{CGE}} F_i (\mathbf{x}; \bc) - \Hat{\nabla}_{\mathrm{CGE}} f_i(\mathbf{x}) \right\|^2 + \mathbb{E}_{\mathcal{I}} (I(i \in \mathcal{I}) - p_i)^2 \left\|  \Hat{\nabla}_{\mathrm{CGE}} f_i(\mathbf{x}) \right\|^2 \right] \nonumber \\
    & \qquad \overset{(k)}{=} \sum_{i=1}^d \frac{1}{p_i} \left[ \mathbb{E}_{\bc} \left\| \Hat{\nabla}_{\mathrm{CGE}} F_i (\mathbf{x}; \bc) - \Hat{\nabla}_{\mathrm{CGE}} f_i(\mathbf{x}) \right\|^2 + (1 - p_i) \left\|  \Hat{\nabla}_{\mathrm{CGE}} f_i(\mathbf{x}) \right\|^2 \right], \label{proof_prop_eq_bd_var_CGE_2}
\end{align}%
where the steps in the derivation of \eqref{proof_prop_eq_bd_var_CGE_2} follow by the reasoning given below:
\begin{itemize}
    \item $(f)$ follows from the definitions of $\Hat{\nabla}_{\mathrm{CGE}} F_{\mathcal{I}} (\mathbf{x}; \bc, \mathbf{p})$ in \eqref{eq_CGE_nonuniform};
    \item $(g)$ follows since $\Hat{\nabla}_{\mathrm{CGE}} F_i (\mathbf{x}_t; \bct), \Hat{\nabla}_{\mathrm{CGE}} f_i(\mathbf{x})$ is aligned with the canonical basis vector $\mbf e_i$, for all $i$;
    \item $(j)$ follows since the stochastic sample set $\bc$ is sampled independent of the coordinate set $\mathcal{I}$. Therefore,
    \begin{align*}
        & \mathbb{E}_{\mathcal{I}, \bc} \left\langle I(i \in \mathcal{I}) \Hat{\nabla}_{\mathrm{CGE}} F_i (\mathbf{x}; \bc) - I(i \in \mathcal{I}) \Hat{\nabla}_{\mathrm{CGE}} f_i(\mathbf{x}), I(i \in \mathcal{I}) \Hat{\nabla}_{\mathrm{CGE}} f_i(\mathbf{x}) - p_i \Hat{\nabla}_{\mathrm{CGE}} f_i(\mathbf{x}) \right\rangle \\
        & \quad = \mathbb{E}_{\mathcal{I}} \left\langle (I(i \in \mathcal{I})) \left( \mathbb{E}_{\bc} \Hat{\nabla}_{\mathrm{CGE}} F_i (\mathbf{x}; \bc) - \Hat{\nabla}_{\mathrm{CGE}} f_i(\mathbf{x}) \right), (I(i \in \mathcal{I}) - p_i) \Hat{\nabla}_{\mathrm{CGE}} f_i(\mathbf{x}) \right\rangle = 0,
    \end{align*}
    \item $(k)$ follows since the indicator function $I(i \in \mathcal{I})$ is a Bernoulli random variable with $\mathrm{Pr} (i \in \mathcal{I}) = p_i$. Therefore, $\mbb E_{\mathcal{I}} [ I^2(i \in \mathcal{I}) ] = 1 \cdot \mathrm{Pr} (i \in \mathcal{I}) + 0 \cdot \mathrm{Pr} (i \notin \mathcal{I}) = p_{i}$. Also, variance of this Bernoulli random variable is given by $\mathrm{var}(I(i \in \mathcal{I})) = \mathbb{E}_{\mathcal{I}} (I(i \in \mathcal{I}) - p_i)^2 = p_i (1-p_i)$.
\end{itemize}

\noindent Next, we bound the first term in \eqref{proof_prop_eq_bd_var_CGE_2} as follows.
\begin{align}
    \mathbb{E}_{\bc} \left\| \Hat{\nabla}_{\mathrm{CGE}} F_i (\mathbf{x}; \bc) - \Hat{\nabla}_{\mathrm{CGE}} f_i (\mathbf{x}) \right\|^2 &= \frac{1}{|\bc|^2} \sum_{\xi \in \bc} \mathbb{E}_{\xi} \left\| \Hat{\nabla}_{\mathrm{CGE}} F_i (\mathbf{x}; \xi) - \Hat{\nabla}_{\mathrm{CGE}} f_i (\mathbf{x}) \right\|^2 \nonumber \\
    &= \frac{1}{|\bc|} \mathbb{E}_{\xi_1} \left\| \Hat{\nabla}_{\mathrm{CGE}} F_i (\mathbf{x}; \xi_1) - \Hat{\nabla}_{\mathrm{CGE}} f_i (\mathbf{x}) \right\|^2 \qquad \text{ for some } \xi_1 \in \bc \nonumber \\
    & \overset{(\ell)}{\leq} \frac{3}{|\bc|} \mathbb{E}_{\xi_1} \left[ \left\| \Hat{\nabla}_{\mathrm{CGE}} F_i (\mathbf{x}; \xi_1) - \mbf e_i \mbf e_i^T \nabla F (\mathbf{x}; \xi_1) \right\|^2 \right. \nonumber \\
    & \left. + \left\| \mbf e_i \mbf e_i^T \left( \nabla F (\mathbf{x}; \xi_1) - \nabla f (\mathbf{x}) \right) \right\|^2  + \left\| \mbf e_i \mbf e_i^T \nabla f (\mathbf{x}) - \Hat{\nabla}_{\mathrm{CGE}} f_i (\mathbf{x}) \right\|^2 \right] \nonumber \\
    & \overset{(m)}{\leq} \frac{3}{|\bc|} \left[ \frac{L^2 \muCi^2}{4} + \zeta^2 + \frac{L^2 \muCi^2}{4} \right] = \frac{3}{|\bc|} \left( \zeta^2 + \frac{L^2 \muCi^2}{2} \right), \label{proof_prop_eq_bd_var_CGE_3}
\end{align}
where $(\ell)$ follows from the inequality $\| \sum_{i=1}^s \mbf x_i \|^2 \leq s \sum_{i=1}^s \| \mbf x_i \|^2$. $(m)$ follows from the coordinate-wise variance bound in Assumption \ref{assum_var_bound} and \cite[Lemma~3.2]{liu2018zeroth}. The second term in \eqref{proof_prop_eq_bd_var_CGE_2} can be bounded as
\begin{align}
    \left\| \Hat{\nabla}_{\mathrm{CGE}} f_i(\mathbf{x}) \right\|^2 & \leq 2 \left\|  \Hat{\nabla}_{\mathrm{CGE}} f_i(\mathbf{x}) - \mbf e_i \mbf e_i^T \nabla f (\mathbf{x}) \right\|^2 + 2 \left\| \mbf e_i \mbf e_i^T \nabla f (\mathbf{x}) \right\|^2 \leq \frac{L^2 \muCi^2}{2} + 2 \left( \nabla f (\mathbf{x}) \right)_i^2, \label{proof_prop_eq_bd_var_CGE_4}
\end{align}
where $(\mbf x)_i$ denotes the $i$-th coordinate of the vector $\mbf x$.
Substituting \eqref{proof_prop_eq_bd_var_CGE_3}, \eqref{proof_prop_eq_bd_var_CGE_4} in \eqref{proof_prop_eq_bd_var_CGE_2}, we get
\begin{align}
    & \mathbb{E} \left\| [\Hat{\nabla}_{\mathrm{CGE}} F_{\mathcal{I}} (\mathbf{x}; \bc, \mathbf{p})] - \Hat{\nabla}_{\mathrm{CGE}} f(\mathbf{x}) \right\|^2 \nonumber \\
    & \qquad \leq \sum_{i=1}^d \frac{1}{p_i} \left[ \frac{3}{|\bc|} \left( \zeta^2 + \frac{L^2 \muCi^2}{2} \right) + (1-p_i) \left\{ \frac{L^2 \muCi^2}{2} + 2 \left( \nabla f (\mathbf{x}) \right)_i^2 \right\} \right] \nonumber \\
    & \qquad = \sum_{i=1}^d \frac{1}{p_i} \left[ \frac{3}{|\bc|} \left( \zeta^2 + \frac{L^2 \muCi^2}{2} \right) + \frac{L^2 \muCi^2}{2} + 2 \left( \nabla f (\mathbf{x}) \right)_i^2 \right] - \frac{L^2}{2} \sum_{i=1}^d \muCi^2 - 2 \left\| \nabla f (\mathbf{x}) \right\|^2. \label{proof_prop_eq_bd_var_CGE_5}
\end{align}
Next, we bound the second term $\left\| \nabla f (\mathbf{x}) - \Hat{\nabla}_{\mathrm{CGE}} f (\mathbf{x}) \right\|^2$ in \eqref{proof_prop_eq_bd_var_CGE_1}. This is done by improving slightly the corresponding result in \cite[Lemma~3]{liu2018zeroth}.
\begin{align}
    \left\| \nabla f (\mathbf{x}) - \Hat{\nabla}_{\mathrm{CGE}} f (\mathbf{x}) \right\|^2 &= \left\| \sum_{i=1}^d \left( \mbf e_i \mbf e_i^T \Hat{\nabla}_{\mathrm{CGE}} f (\mathbf{x}) - (\nabla f (\mathbf{x}))_i \right) \right\|^2 \nonumber \\
    &= \sum_{i=1}^d \left\| \left( \mbf e_i \mbf e_i^T \Hat{\nabla}_{\mathrm{CGE}} f (\mathbf{x}) - (\nabla f (\mathbf{x}))_i \right) \right\|^2 \leq \sum_{i=1}^d \frac{L^2 \muCi^2}{4}, \label{proof_prop_eq_bd_var_CGE_6}
\end{align}
where $(\mbf x)_i$ denotes the $i$-th component of the vector $\mbf x$. The last inequality follows from \cite[Lemma~3]{liu2018zeroth}. Finally, substituting \eqref{proof_prop_eq_bd_var_CGE_5}, \eqref{proof_prop_eq_bd_var_CGE_6} in \eqref{proof_prop_eq_bd_var_CGE_1}, and rearranging the terms, we get
\begin{align}
    \mathbb{E} \left\| [\Hat{\nabla}_{\mathrm{CGE}} F_{\mathcal{I}} (\mathbf{x}; \bc, \mathbf{p})] - \nabla f (\mathbf{x}) \right\|^2 & \leq \sum_{i=1}^d \frac{1}{p_i} \left[ \frac{3}{|\bc|} \left( \zeta^2 + \frac{L^2 \muCi^2}{2} \right) + \frac{L^2 \muCi^2}{2} + 2 \left( \nabla f (\mathbf{x}) \right)_i^2 \right] - 2 \left\| \nabla f (\mathbf{x}) \right\|^2. \label{proof_prop_eq_bd_var_CGE_7}
\end{align}
\end{proof}

\subsection{Choice of $\alpha$}
\label{app_alpha_choice}
Substituting the upper bounds from \eqref{eq_bd_var_RGE}, \eqref{eq_bd_var_CGE} in \eqref{eq_var_HGE}, we get
\begin{align}
    & \mathbb{E} \left\| \Hat{\nabla}_{\mathrm{HGE}} F(\mathbf x; \br, \bc, \mathcal{I}) - \nabla f (\mathbf{x}) \right\|^2 \nonumber \\
    & \quad \leq 2 \alpha^2 \mathbb{E} \left\| \Hat{\nabla}_{\mathrm{RGE}} F (\mathbf x; \br) - \nabla f (\mathbf{x}) \right\|^2 + 2 (1-\alpha)^2 \mathbb{E} \left\|   [\Hat{\nabla}_{\mathrm{CGE}} F_{\mathcal{I}} (\mathbf{x}; \bc, \mathbf{p})] - \nabla f (\mathbf{x}) \right\|^2 \nonumber \\
    & \quad \leq 2 \alpha^2 \left[ \frac{2}{|\br|} \left( 1 + \frac{d}{\nR} \right) \mathbb{E} \left\| \nabla f (\mathbf{x}) \right\|^2 + \frac{2 \sigma^2}{|\br|} \left( 1 + \frac{d}{\nR} \right) + \left( 1 + \frac{2}{|\br|} + \frac{2}{\nR |\br|} \right) \frac{\muR^2 L^2 d^2}{4} \right] \nonumber \\
    & \qquad \qquad + 2 (1-\alpha)^2 \left[ \sum_{i=1}^d \frac{1}{p_i} \Big[ 2 \left( \nabla f (\mathbf{x}) \right)_i^2 + \frac{3}{|\bc|} \left( \zeta^2 + \frac{L^2 \muCi^2}{2} \right) + \frac{L^2 \muCi^2}{2} \Big] - 2 \left\| \nabla f (\mathbf{x}) \right\|^2 \right]. \label{eq_app_alpha_choice_1}
\end{align}
The following steps to choose $\alpha$ are quite similar to those involved in the proof of Theorem \ref{thm_nonconvex}. We describe them here as well, for ease of understanding.
In \eqref{eq_app_alpha_choice_1}, for the time being, we ignore the terms involving $\nabla f(\mbf x)$ (these will be taken care of in Appendix \ref{proof_thm_nonconvex}). 
We assume the CGE smoothing parameters to be constant across coordinates, i.e., $\muCi = \muC$, for all $i \in [d]$.
Using $|\br| \geq 1, |\bc| \geq 1$, from \eqref{eq_app_alpha_choice_1} we can bound the remaining terms to get
\begin{align}
	2 \alpha^2 \left[ 2 \sigma^2 \left( 1 + \frac{d}{\nR} \right) + \left( 3 + \frac{2}{\nR} \right) \frac{\muR^2 L^2 d^2}{4} \right] + 2 (1-\alpha)^2 \left[ \sum_{i=1}^d \frac{1}{p_i} \Big[ 3 \zeta^2 + 2 L^2 \muC^2 \Big] \right]. \label{eq_app_alpha_choice_2}
\end{align}
We denote $\bar{P} = \frac{1}{d} \sum_{i=1}^d \frac{1}{p_{i}}$. Note that $\sigma^2 = d \zeta^2$. 
Also, we choose the smoothing parameters $\muC, \muR$ to be small enough such that $\sigma^2 \geq L^2 \muC^2 d, \sigma^2 \geq L^2 d^2 \muR^2/4$. 
Consequently, 
\begin{align*}
    2 \sigma^2 \left( 1 + \frac{d}{\nR} \right) + \left( 3 + \frac{2}{\nR} \right) \frac{\muR^2 L^2 d^2}{4} & \leq 5 \sigma^2 \left( 1 + \frac{d}{\nR} \right), \\
    \bar{P} \left[ 3 \sigma^2 + 2 L^2 \muC^2 d \right] & \leq 5 \sigma^2 \bar{P}.
\end{align*}
Substituting in \eqref{eq_app_alpha_choice_2}, we get
\begin{align}
	10 \sigma^2 \left[ \alpha^2 \left( 1 + \frac{d}{\nR} \right) + (1-\alpha)^2 \bar{P} \right]. \label{eq_app_alpha_choice_3}
\end{align}
Then, the optimal combination coefficient $\alpha^*$, which minimizes \eqref{eq_app_alpha_choice_3} is given by
\begin{align}
	\alpha^* &= \left[ 1 + \frac{ 1 + \frac{d}{\nR} }{ \bar{P} } \right]^{-1}.
\end{align}

\newpage
\section{Non-convex Case}
\label{app:nonconvex}
\subsection{Proof of Proposition \ref{prop_inner_prod}}
\label{proof_prop_inner_prod}
We state the expressions from Proposition \ref{prop_inner_prod} here for ease of reference.
\begin{align}
    \mathrm{I} & \leq -\frac{3}{4} \left\| \nabla f (\mathbf{x}_{t}) \right\|^2 + \left( \dfrac{\muR d L}{2} \right)^2, \nonumber \\
    \mathrm{II} & \leq -\frac{3}{4} \left\| \nabla f (\mathbf{x}_{t}) \right\|^2 + L^2 d \muC^2. \nonumber
\end{align}
\begin{proof}
The inner product term $\mathrm{I}$ in \eqref{eq_f_smoothness_exp1} can be bounded as follows.
\begin{align}
	\mathrm{I} = - \left\langle \nabla f (\mathbf{x}_{t}), \nabla f_{\muR} (\mathbf{x}_{t}) \right\rangle &= - \left[ \frac{\left\| \nabla f (\mathbf{x}_{t}) \right\|^2 + \left\| \nabla f_{\muR} (\mathbf{x}_{t}) \right\|^2 - \left\| \nabla f (\mathbf{x}_{t}) - \nabla f_{\muR} (\mathbf{x}_{t}) \right\|^2}{2} \right] \nonumber \\
	& \overset{(a)}{\leq} -\frac{\left\| \nabla f (\mathbf{x}_{t}) \right\|^2}{2} + \frac{1}{2} \left( \left\| \nabla f (\mathbf{x}_{t}) - \nabla f_{\muR} (\mathbf{x}_{t}) \right\|^2 - \frac{\left\| \nabla f (\mathbf{x}_{t}) \right\|^2}{2} \right) + \frac{1}{2} \left( \frac{\muR d L}{2} \right)^2 \nonumber \\
	& \leq -\frac{3}{4} \left\| \nabla f (\mathbf{x}_{t}) \right\|^2 + \frac{\muR^2 d^2 L^2}{4} \label{proof_eq_inner_prod_bd1}
\end{align}
where $(a)$ follows from the inequalities (i) $\| \mathbf{a} \|^2 \geq \frac{\| \mathbf{b} \|^2}{2} - \| \mathbf{a} - \mathbf{b} \|^2$, and (ii) $\| \nabla f (\mathbf{x}) - \nabla f_{\muR} (\mathbf{x}) \| \leq \frac{\muR d L}{2}$ \cite[Lemma~4.1]{gao18information}.
Similarly, we can upper bound $\mathrm{II}$ in \eqref{eq_f_smoothness_exp1}.
\begin{align}
	\mathrm{II} = -\left\langle \nabla f (\mathbf{x}_{t}), \Hat{\nabla}_{\mathrm{CGE}} f (\mathbf{x}_{t}) \right\rangle & \leq -\frac{3}{4} \left\| \nabla f (\mathbf{x}_{t}) \right\|^2 + \left\| \nabla f (\mathbf{x}_{t}) - \Hat{\nabla}_{\mathrm{CGE}} f (\mathbf{x}_{t}) \right\|^2 \nonumber \\
	& \overset{(b)}{\leq} -\frac{3}{4} \left\| \nabla f (\mathbf{x}_{t}) \right\|^2 + L^2 d \muC^2, \label{proof_eq_inner_prod_bd2}
\end{align}
where $(b)$ follows from \cite[Lemma~3]{zo_spider_liang19icml}.
\end{proof}

\subsection{Proof of Proposition \ref{prop_bd_norm_sq_rge}}
\label{proof_prop_bd_norm_sq_rge}
We state the expressions from Proposition \ref{prop_bd_norm_sq_rge} here for ease of reference.
\begin{align}
	\mathrm{III} &= \mathbb{E}_{Y_t} \left[ \left\| \nabla_{\mathrm{r},t} \right\|^2 \mid \mathcal{Y}_t \right] \leq 2 \left\| \nabla f (\mathbf{x}_t) \right\|^2 + \frac{4}{|\brt|} \left( 1 + \frac{d}{\nR} \right) \left[ \left\| \nabla f (\mathbf{x}) \right\|^2 + \sigma^2 \right] + \left( 1 + \frac{2}{|\brt|} + \frac{2}{\nR |\brt|} \right) \frac{\muR^2 L^2 d^2}{2}, \\
	\mathrm{IV} &= \mathbb{E}_{Y_t} \left[ \left\| \nabla_{\mathrm{c},t} \right\|^2 \mid \mathcal{Y}_t \right] \leq \sum_{i=1}^d \frac{1}{p_{t,i}} \left[ \frac{3 \zeta^2}{|\bct|} + \frac{L^2 \muC^2}{2} \left( 1 + \frac{3}{|\bct|} \right) + 2 \left( \nabla f (\mathbf{x}_{t}) \right)_i^2 \right].
\end{align}

\begin{proof}
First, we prove the bound on $\mathrm{III}$.
\begin{align}
	\mathrm{III} = \mathbb{E}_{Y_t} \left[ \left\| \nabla_{\mathrm{r},t} \right\|^2 \mid \mathcal{Y}_t \right] &= \mathbb{E} \left[ \left\| \nabla_{\mathrm{r},t} - \nabla f (\mathbf{x}_t) + \nabla f (\mathbf{x}_t) \right\|^2 \big| \mathcal{Y}_t \right] \nonumber \\
	&\overset{(a)}{\leq} 2 \mathbb{E} \left[ \left\| \nabla_{\mathrm{r},t} - \nabla f (\mathbf{x}_t) \right\|^2 \big| \mathcal{Y}_t \right] + 2 \left\| \nabla f (\mathbf{x}_t) \right\|^2 \nonumber \\
	& \overset{(b)}{\leq} \frac{4}{|\brt|} \left( 1 + \frac{d}{\nR} \right) \left[ \left\| \nabla f (\mathbf{x}) \right\|^2 + \sigma^2 \right] + \left( 1 + \frac{2}{|\brt|} + \frac{2}{\nR |\brt|} \right) \frac{\muR^2 L^2 d^2}{2} + 2 \left\| \nabla f (\mathbf{x}_t) \right\|^2, \nonumber
\end{align}
where $(a)$ follows from $\| \mbf a + \mbf b \|^2 \leq 2 \| \mbf a \|^2 + 2 \| b \|^2$; $(b)$ follows from \eqref{eq_bd_var_RGE} in Proposition \ref{prop_bd_var_RGE_CGE}, proved in Appendix \ref{proof_prop_bd_var_RGE_CGE}. Next, we bound $\mathrm{IV}$. Here, rather than using the bound in Proposition \ref{prop_bd_var_RGE_CGE}, we follow a slightly different route to derive a tighter bound.
\begin{align}
	\mathrm{IV} = \mathbb{E} \left[ \left\| \nabla_{\mathrm{c},t} \right\|^2 \mid \mathcal{Y}_t \right] &= \mathbb{E} \left[ \left\| \sum_{i=1}^d \frac{I(i \in \mathcal{I}_t)}{p_{t,i}} \Hat{\nabla}_{\mathrm{CGE}} F_i (\mathbf{x}_t; \bct) \right\|^2 \Big| \mathcal{Y}_t \right] \qquad \qquad (\text{see } \eqref{eq_CGE_nonuniform}) \nonumber \\
	&\overset{(c)}{=} \sum_{i=1}^d \mathbb{E} \left[ \left\| \frac{I(i \in \mathcal{I}_t)}{p_{t,i}} \Hat{\nabla}_{\mathrm{CGE}} F_i (\mathbf{x}_t; \bct) \right\|^2 \Big| \mathcal{Y}_t \right] \nonumber \\
	&= \sum_{i=1}^d \frac{1}{p_{t,i}^2} \left[ \mathbb{E}_{\mathcal{I}_t} \left( I(i \in \mathcal{I}_t) \right)^2 \mathbb{E}_{\bct} \left\| \frac{1}{|\bct|} \sum_{\xi \in \bct} \Hat{\nabla}_{\mathrm{CGE}} F_i (\mathbf{x}_t; \xi) \right\|^2 \Big| \mathcal{Y}_t \right] \nonumber \\
	& \overset{(d)}{=} \sum_{i=1}^d p_{t,i} \frac{1}{p_{t,i}^2} \mathbb{E} \left[ \frac{1}{|\bct|} \left\| \Hat{\nabla}_{\mathrm{CGE}} F_i (\mathbf{x}_t; \xi_1) - \Hat{\nabla}_{\mathrm{CGE}} f_i (\mathbf{x}_t) \right\|^2 + \left\| \Hat{\nabla}_{\mathrm{CGE}} f_i (\mathbf{x}_t) \right\|^2 \Big| \mathcal{Y}_t \right]. \label{proof_eq_prop_bd_norm_sq_rge_1}
\end{align}
where $(c)$ follows from \eqref{eq_CGE_1} since $\Hat{\nabla}_{\mathrm{CGE}} F_i (\mathbf{x}_t; \bct)$ is aligned with the canonical basis vector $\mbf e_i$, for all $i$. $(d)$ follows since $\mbb E [ (I(i \in \mathcal{I}_t))^2 ] = p_{t,i}$ as seen in Appendix \ref{proof_prop_bd_var_RGE_CGE}. Also, $\mathcal{I}_t$ is independent of $\bct$, and the elements of $\bct$ are again sampled independent of each other, and since $\mbb E_{\bct} \Hat{\nabla}_{\mathrm{CGE}} F_i (\mathbf{x}_t; \bct) = \Hat{\nabla}_{\mathrm{CGE}} f_i (\mathbf{x}_t)$.
Next, we upper bound the two terms in \eqref{proof_eq_prop_bd_norm_sq_rge_1}.
\begin{align}
    \mathbb{E} \left\| \Hat{\nabla}_{\mathrm{CGE}} F_i (\mathbf{x}_t; \xi_1) - \Hat{\nabla}_{\mathrm{CGE}} f_i (\mathbf{x}_t) \right\|^2 \leq 3 \left[ \zeta^2 + \frac{L^2 \muCi^2}{2} \right], \label{proof_eq_prop_bd_norm_sq_rge_2}
\end{align}
follows from \eqref{proof_prop_eq_bd_var_CGE_3}.
The second term in \eqref{proof_eq_prop_bd_norm_sq_rge_1} can be bounded using \eqref{proof_prop_eq_bd_var_CGE_4}, 
\begin{align}
    \left\| \Hat{\nabla}_{\mathrm{CGE}} f_i (\mathbf{x}_t) \right\|^2 & \leq \frac{L^2 \muCi^2}{2} + 2 \left( \nabla f (\mathbf{x}_{t}) \right)_i^2, \label{proof_eq_prop_bd_norm_sq_rge_3}
\end{align}
where $(\mbf x)_i$ denotes the $i$-th coordinate of the vector $\mbf x$. Substituting \eqref{proof_eq_prop_bd_norm_sq_rge_2}, \eqref{proof_eq_prop_bd_norm_sq_rge_3} in \eqref{proof_eq_prop_bd_norm_sq_rge_1}, we get
\begin{align}
	\mathbb{E} \left[ \left\| \nabla_{\mathrm{c},t} \right\|^2 \mid \mathcal{Y}_t \right] \leq \sum_{i=1}^d \frac{1}{p_{t,i}} \left[ \frac{3 \zeta^2}{|\bct|} + \frac{L^2 \muCi^2}{2} \left( 1 + \frac{3}{|\bct|} \right) + 2 \left( \nabla f (\mathbf{x}_{t}) \right)_i^2 \right]. \nonumber
\end{align}
Taking $\muCi = \muC$, for all $i \in [d]$, we get the bound on $\mathrm{IV}$.
\end{proof}

\subsection{Proof of Theorem \ref{thm_nonconvex}}
\label{proof_thm_nonconvex}
\begin{proof}
Substituting the bounds on $\mathrm{I, II}$ from Proposition \ref{prop_inner_prod}, and the bounds on $\mathrm{III, IV}$ from Proposition \ref{prop_bd_norm_sq_rge} in \eqref{eq_f_smoothness_exp1}, and taking expectation over the entire randomness till time $t$, we get
\begin{align}
	\mathbb{E} f (\mathbf{x}_{t+1}) & \leq \mathbb{E} f (\mathbf{x}_{t}) -\frac{3\eta_t \alpha_t}{4} \mathbb{E} \left\| \nabla f (\mathbf{x}_{t}) \right\|^2 + \eta_t \alpha_t \frac{\muR^2 d^2 L^2}{4} -\frac{3\eta_t (1 - \alpha_t)}{4} \mathbb{E} \left\| \nabla f (\mathbf{x}_{t}) \right\|^2 + \eta_t (1 - \alpha_t) L^2 d \muC^2 \nonumber \\
	& \qquad + \eta_t^2 L \alpha_t^2 \left[ 2 \left( 1 + \frac{2}{|\brt|} + \frac{2 d}{\nR |\brt|} \right) \left\| \nabla f (\mathbf{x}) \right\|^2 + \frac{4 \sigma^2}{|\brt|} \left( 1 + \frac{d}{\nR} \right) + \left( 1 + \frac{2}{|\brt|} + \frac{2}{\nR |\brt|} \right) \frac{\muR^2 L^2 d^2}{2} \right] \nonumber \\
	& \qquad + \eta_t^2 L (1-\alpha_t)^2 \sum_{i=1}^d \frac{1}{p_{t,i}} \left[  \frac{3}{|\bct|} \left( \zeta^2 + \frac{L^2 \muCi^2}{2} \right) + \frac{L^2 \muCi^2}{2} + 2 \mbb E \left( \nabla f (\mathbf{x}_{t}) \right)_i^2 \right]. \label{proof_eq_thm_nonconvex_1}
\end{align}
Note that $\sigma^2 = d \zeta^2$. Using $p_{t,i} \geq c_t$ for all $i \in [d]$, and $|\brt| \geq 1, |\bct| \geq 1, \nR \geq 1$, \eqref{proof_eq_thm_nonconvex_1} leads to
\begin{align}
	\mathbb{E} f (\mathbf{x}_{t+1}) & \leq \mathbb{E} f (\mathbf{x}_{t}) - \eta_t \left\{ \frac{3}{4} - 6 \eta_t L \alpha_t^2 \left( 1 + \frac{d}{\nR} \right) - 2 L \frac{\eta_t (1 - \alpha_t)^2}{c_t} \right\} \mathbb{E} \left\| \nabla f (\mathbf{x}_{t}) \right\|^2 + \eta_t \alpha_t \frac{L^2 \muR^2 d^2}{4} \nonumber \\
	& + \eta_t (1 - \alpha_t) L^2 d \muC^2 + 4 \eta_t^2 L \alpha_t^2 \left( 1 + \frac{d}{\nR} \right) \left[ \muR^2 d^2 L^2 + \sigma^2  \right] + 4 L \eta_t^2 (1 - \alpha_t)^2 \sum_{i=1}^d \frac{1}{p_{t,i}} \left[ \zeta^2 + L^2 \muC^2 \right]. \label{proof_eq_thm_nonconvex_2}
\end{align}
Henceforth, we assume constant step-sizes $\eta_t= \eta$ for all $t$, and constant combination coefficients $\alpha_t = \alpha$, for all $t$. We denote $P_t = \frac{1}{d} \sum_{i=1}^d \frac{1}{p_{t,i}}$. 
Rearranging the terms in \eqref{proof_eq_thm_nonconvex_2}, summing over $t = 0$ to $T-1$, and dividing by $\eta T$ we get
\begin{align}
	\frac{1}{T} \sum_{t=0}^{T-1} & \left\{ \frac{3}{4} - \eta L 6 \alpha^2 \left( 1 + \frac{d}{\nR} \right) - 2 L \frac{\eta (1 - \alpha)^2}{c_t} \right\} \mathbb{E} \left\| \nabla f (\mathbf{x}_{t}) \right\|^2 \leq \frac{f (\mathbf{x}_{0}) - \mathbb{E} f (\mathbf{x}_{T})}{\eta T} + L^2 d \muC^2 \nonumber \\
	& + \alpha \left( \frac{L^2 \muR^2 d^2}{4} - L^2 d \muC^2 \right) + 4 \eta L \alpha^2 \left( 1 + \frac{d}{\nR} \right) \left[ \muR^2 d^2 L^2 + \sigma^2 \right] + 4 L \eta (1 - \alpha)^2 \left[ \sigma^2 + L^2 \muC^2 d \right] \frac{1}{T} \sum_{t=0}^{T-1} P_t. \label{proof_eq_thm_nonconvex_3}
\end{align}

\noindent To ease the notation, we define the following constants.
\begin{align*}
	A = 4 \sigma^2 + 4 L^2 \muC^2 d, \quad  C = 4 \left( 1 + \frac{d}{\nR} \right) \left[ \muR^2 d^2 L^2 + \sigma^2 \right], \quad (\Delta f) = f(\mathbf{x}_0) - f^*, \quad \bar{P}_T =  \frac{1}{T} \sum_{t=0}^{T-1} P_t.
\end{align*}
Further, we select the smoothing parameters $\muC, \muR$ such that $\frac{L^2 \muR^2 d^2}{4} = L^2 d \muC^2$. This simplifies \eqref{proof_eq_thm_nonconvex_3} to
\begin{align}
	\frac{1}{T} \sum_{t=0}^{T-1} & \left\{ \frac{3}{4} - \eta L 4 \alpha^2 \left( 1 + \frac{d}{\nR} \right) - 2 L \frac{\eta (1 - \alpha)^2}{c_t} \right\} \mathbb{E} \left\| \nabla f (\mathbf{x}_{t}) \right\|^2 \leq \frac{(\Delta f)}{\eta T} + L^2 d \muC^2 + \alpha^2 \eta L C + L \eta (1 - \alpha)^2 A \bar{P}_T. \label{proof_eq_thm_nonconvex_4}
\end{align}
We choose $\eta$ such that $\frac{3}{4} - \eta L 6  \alpha^2 \left( 1 + \frac{d}{\nR} \right) - 2 L \frac{\eta (1 - \alpha)^2}{c_t} \geq \frac{1}{2}, \forall \ t$. This leads to
\begin{align}
	\eta &\leq \frac{1}{8L} \left[ 3 \alpha^2 \left( 1 + \frac{d}{\nR} \right) + \frac{(1 - \alpha)^2}{c_t} \right]^{-1} \nonumber \\
	\Rightarrow \eta &\leq \frac{1}{24 L} \min \left\{ 3 c_t, \frac{\nR}{d + \nR} \right\}, \forall \ t, \label{eq_bound_stepsize}
\end{align}
where \eqref{eq_bound_stepsize} follows since $\left[ 3 \alpha^2 \left( 1 + \frac{d}{\nR} \right) + \frac{(1 - \alpha)^2}{c_t} \right]^{-1}$ must attain its minimum value over $[0,1]$ at one of the end points.
Optimizing for $\eta$ in \eqref{proof_eq_thm_nonconvex_4} while satisfying \eqref{eq_bound_stepsize}, we get
\begin{align}
	\frac{1}{2T} \sum_{t=0}^{T-1} \mathbb{E} \left\| \nabla f (\mathbf{x}_{t}) \right\|^2 & \leq 2 \sqrt{ \frac{(\Delta f) L}{T} \left( \alpha^2 C + (1 - \alpha)^2 A \bar{P}_T \right)} + L^2 d \muC^2. \label{proof_eq_thm_nonconvex_5}
\end{align}
Note that in \eqref{proof_eq_thm_nonconvex_5}, $L^2 d \muC^2$ needs to be small to ensure convergence, and the smoothness parameter $\muC$ can be designed for that purpose. 
However, $\sigma^2$ being the variance, is not in our control. We choose the smoothing parameters $\muC, \muR$ to be small enough such that $\sigma^2 \geq L^2 \muC^2 d = L^2 d^2 \muR^2/4$. 
Consequently, $A \leq 8 \sigma^2$, $C \leq 8 \sigma^2 \left( 1 + \frac{d}{\nR} \right)$. Substituting in \eqref{proof_eq_thm_nonconvex_5}, we get
\begin{align}
	\frac{1}{2T} \sum_{t=0}^{T-1} \mathbb{E} \left\| \nabla f (\mathbf{x}_{t}) \right\|^2 & \leq 2 \sqrt{ \frac{(\Delta f) L}{T} 8 \sigma^2 \left( \alpha^2 \left( 1 + \frac{d}{\nR} \right) + (1 - \alpha)^2 \bar{P}_T \right)} + L^2 d \muC^2.
	\label{proof_eq_thm_nonconvex_6}
\end{align}
Then, the optimal combination coefficient $\alpha^*$ is given by
\begin{align}
	\alpha^* &= \left[ 1 + \frac{ 1 + \frac{d}{\nR} }{ \bar{P}_T } \right]^{-1}. \label{proof_eq_thm_nonconvex_7}
\end{align}
Substituting \eqref{proof_eq_thm_nonconvex_7} in \eqref{proof_eq_thm_nonconvex_6}, we get
\begin{align}
	\frac{1}{2T} \sum_{t=0}^{T-1} \mathbb{E} \left\| \nabla f (\mathbf{x}_{t}) \right\|^2 & \leq 2 \sqrt{ \frac{(\Delta f) L}{T} 8 \sigma^2 \left( \dfrac{1 + \frac{d}{\nR} }{1 + \frac{ 1 + \frac{d}{\nR} }{ \bar{P}_T }} \right)} + L^2 d \muC^2. \label{proof_eq_thm_nonconvex_8}
\end{align}
Since $\sigma^2 \geq L^2 \muR^2 d^2$, we choose $\muC$ such that
\begin{align*}
	L^2 d \muC^2 = \mathcal{O} \left( \sqrt{ \frac{(1 + d/\nR)}{T} \dfrac{1}{1 + \frac{(1 + d/\nR)}{\bar{P}_T}} } \right) \quad \Rightarrow \quad \muC = \mathcal{O} \left( \left( \frac{(1 + d/\nR)}{d^2 T} \right)^{1/4} \dfrac{1}{ \left( 1 + \frac{(1 + d/\nR)}{\bar{P}_T} \right)^{1/4}} \right).
\end{align*}
Substituting $\muC, \muR = \frac{2 \muC}{\sqrt{d}}$ in \eqref{proof_eq_thm_nonconvex_8} gives us
\begin{align}
    \mathbb{E} \left\| \nabla f (\bar{\mathbf{x}}_T) \right\|^2 \leq \frac{1}{T} \sum_{t=0}^{T-1} \mathbb{E} \left\| \nabla f (\mathbf{x}_{t}) \right\|^2 & \leq \mathcal{O} \left( \sqrt{ \frac{(1 + d/\nR)}{T} \dfrac{1}{1 + \frac{(1 + d/\nR)}{\bar{P}_T}} } \right), \label{proof_eq_thm_nonconvex_9}
\end{align}
where the first inequality follows from Jensen's inequality. 
This completes the proof.
\end{proof}

\subsection{Special Cases of Theorem \ref{thm_nonconvex}}
\label{app_special_case_nonconvex}

\subsubsection{Regime 1: $\dnr = 1 + \frac{d}{\nR} \gg \bar{P}_T$.}
Since $\bar{P}_T \geq 1$, this implies that $\frac{d}{\nR} \gg \bar{P}_T$. Note that since $\bar{P}_T$ is the mean of inverse probabilities values $\{ 1/p_{t,i} \}$ across time $t$ and dimensions $i$, $\frac{1}{\bar{P}_T} \gg \frac{\nR}{d}$ implies that on average, sampling probabilities are much greater than $\nR/d$. 
In other words, the per-iteration query budget of CGE is much higher than RGE, and $\alpha \to 0$ (see \eqref{proof_eq_thm_nonconvex_7}). 
With smoothing parameter $\muC = O \left( (\bar{P}_T/d^2 T)^{1/4} \right)$, the resulting convergence rate is
\begin{align*}
	\mathbb{E} \left\| \nabla f (\bar{\mathbf{x}}_T) \right\|^2 \leq O \left(  \sqrt{ \frac{\bar{P}_T}{T}} \right).
\end{align*}
FQC to achieve $\mathbb{E} \| \nabla f (\bar{\mathbf{x}}_T) \|^2 \leq \epsilon$ is given by $O(T \cdot\nR+ \sum_{t=0}^{T-1} \sum_{i=1}^d p_{t,i}) = O(T \cdot (\nR+ \nC))$. 

In the special case of uniform distribution for CGE, i.e., $p_{t,i} = \nC/d$, $\bar{P}_T = d/\nC$.
Consequently, the convergence rate is
\begin{align*}
    \mathbb{E} \| \nabla f (\bar{\mathbf{x}}_T) \|^2 \leq O (  \sqrt{ d/(\nC T)} ).
\end{align*}
Also, $\frac{d}{\nR} \gg \bar{P}_T$ implies $\nC \gg \nR$.
Hence, FQC to achieve $\mathbb{E} \| \nabla f (\bar{\mathbf{x}}_T) \|^2 \leq \epsilon$ is $O(T \cdot \nC) = O(d/\epsilon^2)$.
Naturally, both the convergence rate and FQC are dominated by CGE.
For $\nC = 1$ the performance reduces to that of ZO-SCD (see Table \ref{table_comparison}).

\subsubsection{Regime 2: $\dnr = 1 + \frac{d}{\nR} \ll \bar{P}_T$.}
Since $\bar{P}_T \geq 1$, this implies that $\frac{d}{\nR} \ll \bar{P}_T$ and $\frac{1}{\bar{P}_T} \ll \frac{\nR}{d}$ implies that on average, sampling probabilities are much smaller than $\nR/d$. 
In other words, the per-iteration query budget of CGE is much smaller than RGE, and $\alpha \to 1$ (see \eqref{proof_eq_thm_nonconvex_7}). 
With smoothing parameters $\muC = O \left( (\dnr/d^2 T)^{1/4} \right)$, $\muR = O \left( d^{-1} (\dnr/T)^{1/4} \right)$, the resulting convergence rate is
\begin{align*}
	\mathbb{E} \left\| \nabla f (\bar{\mathbf{x}}_T) \right\|^2 \leq O \left(  \sqrt{ \frac{1 + \frac{d}{\nR}}{T}} \right).
\end{align*}
FQC to achieve $\mathbb{E} \| \nabla f (\bar{\mathbf{x}}_T) \|^2 \leq \epsilon$ is given by $O(T \cdot\nR+ \sum_{t=0}^{T-1} \sum_{i=1}^d p_{t,i}) = O(T \cdot (\nR+ \nC))$. 
In the special case of uniform distribution for CGE, i.e., $p_{t,i} = \nC/d$, $\bar{P}_T = d/\nC$.
$\frac{d}{\nR} \ll \bar{P}_T$ implies $\nC \ll \nR$.
Hence, FQC to achieve $\mathbb{E} \| \nabla f (\bar{\mathbf{x}}_T) \|^2 \leq \epsilon$ is $O(T \cdot \nR) = O(d/\epsilon^2)$ (assuming $\nR = O(d)$).
Naturally, both the convergence rate and FQC are dominated by RGE.
Also, note that compared to ZO-SGD, we achieve the same convergence rate, while the bound on $\muR$ is more relaxed (see Table \ref{table_comparison}).

\subsubsection{Regime 3: $\dnr = 1 + \frac{d}{\nR}$ and $\bar{P}_T$ are comparable in value.}
To gain some insight into this case where the function query budgets of RGE and CGE are comparable, we again look at the uniform distribution $p_{t,i} = \frac{\nC}{d}, \forall \ t,i$. So, $\bar{P}_T = \frac{d}{\nC}$.
The total per-iteration function query cost of HGE is $O(\nR + \nC)$.
First, note that
\begin{align}
	\frac{1 + \frac{d}{\nR}}{1 + \frac{1 + \frac{d}{\nR}}{\bar{P}_T}} &= \frac{1 + \frac{d}{\nR}}{1 + \left( 1 + \frac{d}{\nR} \right) \frac{\nC}{d}} = \frac{d \nR + d^2}{\nR \nC + d(\nR + \nC)} \nonumber \\
	&\leq \frac{\nR}{\nR + \nC} + \frac{d}{\nR + \nC} \leq 1 + \frac{d}{\nR + \nC}. \nonumber
\end{align}
With $\muC = O \left(  \left( \frac{1}{d^2 T} \left( 1 + \frac{d}{\nR + \nC} \right) \right)^{1/4} \right)$, $\muR = O \left( \frac{1}{d} \left( \frac{1}{T} \left( 1 + \frac{d}{\nR + \nC} \right) \right)^{1/4} \right)$, the resulting convergence rate is
\begin{align*}
	\mathbb{E} \left\| \nabla f (\bar{\mathbf{x}}_T) \right\|^2 \leq O \left( \sqrt{\frac{1}{T} \left( 1 + \frac{d}{\nR + \nC} \right)} \right).
\end{align*}
FQC to achieve $\mathbb{E} \| \nabla f (\bar{\mathbf{x}}_T) \|^2 \leq \epsilon$ is given by $ O(T \cdot\nR+ T \cdot \nC) = O(d/\epsilon^2)$ (assuming $\nR + \nC = O(d)$).

\newpage
\section{Convex Case}
\label{app_convex}
Before, proceeding with the proof of Theorem \ref{thm_convex}, we prove some intermediate results, which shall be used along the way. First, using convexity of $f$ (Assumption \ref{assum_convexity})
\begin{align}
	& \sum_{t=1}^T \left( f(\mathbf{x}_t) - f(\mathbf{x}^*) \right) \leq \sum_{t=1}^T \left\langle \nabla f(\mathbf{x}_t), \mathbf{x}_t - \mathbf{x}^* \right\rangle \nonumber \\
	&= \sum_{t=1}^T \left\langle \nabla_t, \mathbf{x}_t - \mathbf{x}^* \right\rangle + \sum_{t=1}^T \left\langle \nabla f(\mathbf{x}_t) - \nabla_t, \mathbf{x}_t - \mathbf{x}^* \right\rangle \label{eq_convexity}
\end{align}
where $x^* = \argmin_{\mathbf{x} \in \mathrm{dom } f} f(\mbf x)$, and we denote the descent direction used in Algorithm \ref{Algo_zo_hgd} as $\nabla_t \triangleq \alpha_t \nabla_{\mathrm{r},t} + (1-\alpha_t) \nabla_{\mathrm{c},t}$. Next, we bound both the terms in \eqref{eq_convexity} separately in the following two results.

\begin{prop}
\label{prop_convex_inner_prod_1}
Under Assumption \ref{assum_convexity}, and using non-increasing step-sizes $\{ \eta_t \}$ in Algorithm \ref{Algo_zo_hgd}
\begin{align}
	\sum_{t=1}^T \left\langle \nabla_t, \mathbf{x}_t - \mathbf{x}^* \right\rangle &\leq \frac{R^2}{2 \eta_T} + \sum_{t=1}^T \frac{\eta_t}{2} \left\| \nabla_t \right\|^2 \label{eq_prop_convex_inner_prod_1}
\end{align}
\end{prop}

\begin{proof}
The results follows by a straightforward application of the Young's inequality, and using Assumption \ref{assum_convexity} to bound $\| \mathbf{x}_t - \mathbf{x}^* \|$ with $R$.
\end{proof}

\begin{prop}
\label{prop_convex_inner_prod_2}
\begin{align}
	\sum_{t=1}^T \mathbb{E} \left\langle \nabla f(\mathbf{x}_t) - \nabla_t, \mathbf{x}_t - \mathbf{x}^* \right\rangle \leq R L \sum_{t=1}^T \left( \frac{\alpha_t \muR d}{2} + (1-\alpha_t) \sqrt{d} \muC \right). \label{eq_prop_convex_inner_prod_2}
\end{align}
\end{prop}

\begin{proof}
\begin{align}
	& \sum_{t=1}^T \mathbb{E} \left\langle \nabla f(\mathbf{x}_t) - \nabla_t, \mathbf{x}_t - \mathbf{x}^* \right\rangle \nonumber \\
	& \qquad = \sum_{t=1}^T \left[ \alpha_t \mathbb{E} \left[ \left\langle \nabla f(\mathbf{x}_t) - \nabla_{\mathrm{r},t}, \mathbf{x}_t - \mathbf{x}^* \right\rangle \mid \mathcal{Y}_t \right] + (1-\alpha_t) \mathbb{E} \left[ \left\langle \nabla f(\mathbf{x}_t) - \nabla_{\mathrm{c},t}, \mathbf{x}_t - \mathbf{x}^* \right\rangle \mid \mathcal{Y}_t \right] \right] \nonumber \\
	& \qquad \overset{(a)}{=} \sum_{t=1}^T \left[ \alpha_t \mathbb{E} \left\langle \nabla f(\mathbf{x}_t) - \nabla f_{\muR} (\mathbf{x}_t), \mathbf{x}_t - \mathbf{x}^* \right\rangle + (1-\alpha_t) \mathbb{E} \left\langle \nabla f(\mathbf{x}_t) - \Hat{\nabla}_{\mathrm{CGE}} f(\mathbf{x}_t), \mathbf{x}_t - \mathbf{x}^* \right\rangle \right] \nonumber \\
	& \qquad \overset{(b)}{\leq} \sum_{t=1}^T \left[ \alpha_t \mathbb{E} \left\| \mathbf{x}_t - \mathbf{x}^* \right\|_2 \left( \frac{\muR L d}{2} \right) + (1-\alpha_t) \mathbb{E} \left\| \mathbf{x}_t - \mathbf{x}^* \right\|_2 L \sqrt{d} \muC \right], \nonumber
\end{align}
where, $(a)$ follows from $\mathbb{E} \nabla_{\mathrm{r},t} = \nabla f_{\muR} (\mathbf{x}_t)$, and $\mathbb{E} \nabla_{\mathrm{c},t} = \Hat{\nabla}_{\mathrm{CGE}} f(\mathbf{x}_t)$. $(b)$ follows from Cauch-Schwarz inequality, and by using the bounds $\| \nabla f(\mathbf{x}_t) - \nabla f_{\muR} (\mathbf{x}_t) \| \leq \frac{\muR L d}{2}$ \cite[Lemma~4.1]{gao18information}, and $\| \nabla f(\mathbf{x}_t) - \Hat{\nabla}_{\mathrm{CGE}} f(\mathbf{x}_t) \| \leq L \sqrt{d} \muC$ \cite[Lemma~3]{zo_spider_liang19icml}. The result follows since $\left\| \mathbf{x}_t - \mathbf{x}^* \right\|_2 \leq R$ from Assumption \ref{assum_convexity}.
\end{proof}

\noindent Next we bound $\left\| \nabla_t \right\|^2$ which appears in Proposition \ref{prop_convex_inner_prod_1}. 
We make use of Assumption \ref{assum_bound_grad} for this.
\begin{prop}
\label{prop_convex_norm_bd}
\begin{align}
    \mathbb{E} \left\| \nabla_t \right\|^2 & \leq 2 \alpha_t^2 \left[ \left( 2 + \frac{4}{|\br|} \left( 1 + \frac{d}{\nR} \right) \right) \left\| \nabla f (\mathbf{x}) \right\|^2 + \frac{4 \sigma^2}{|\br|} \left( 1 + \frac{d}{\nR} \right) + \left( 1 + \frac{2}{|\br|} + \frac{2}{\nR |\br|} \right) \frac{\muR^2 L^2 d^2}{2} \right] \nonumber \\
    & \qquad + 2 (1-\alpha_t)^2 \sum_{i=1}^d \frac{1}{p_{t,i}} \left[ \frac{3 \zeta^2}{|\bct|} + \frac{L^2 \muC^2}{2} \left( 1 + \frac{3}{|\bct|} \right) + 2 \left( \nabla f (\mathbf{x}_{t}) \right)_i^2 \right]. \nonumber
\end{align}
\end{prop}

\begin{proof}
The proof follows by substituting the bounds on $\mathbb{E} \left\| \nabla_{\mathrm{r},t} \right\|^2$ and $\mathbb{E} \left\| \nabla_{\mathrm{c},t} \right\|^2$ from Proposition \ref{prop_bd_norm_sq_rge}.
\end{proof}

\subsection{Proof of Theorem \ref{thm_convex} (Convex Case)}
\label{proof_thm_convex}
The set of coordinate-wise probabilities $\{ p_{t,i} \}$ are chosen such that $p_{t,i} \geq \bar{c}, \ \forall \ i, t$. 

\begin{proof}
Using Proposition \ref{prop_convex_inner_prod_1}, \ref{prop_convex_inner_prod_2}, \ref{prop_convex_norm_bd} in \eqref{eq_convexity}, we get
\begin{align}
	\sum_{t=1}^T & \left( \mathbb{E} f(\mathbf{x}_t) - f(\mathbf{x}^*) \right) \leq \frac{R^2}{2 \eta_T} + R L \sum_{t=1}^T \left( \frac{\alpha_t \muR d}{2} + (1-\alpha_t) \sqrt{d} \muC \right) \nonumber \\
	& \quad + \sum_{t=1}^T \eta_t (1-\alpha_t)^2 \sum_{i=1}^d \frac{1}{p_{t,i}} \left[ \frac{3 \zeta^2}{|\bct|} + \frac{L^2 \muC^2}{2} \left( 1 + \frac{3}{|\bct|} \right) + 2 \left( \nabla f (\mathbf{x}_{t}) \right)_i^2 \right] \nonumber \\
	& \quad + \sum_{t=1}^T \eta_t \alpha_t^2 \left[ \left( 2 + \frac{4}{|\br|} \left( 1 + \frac{d}{\nR} \right) \right) \left\| \nabla f (\mathbf{x}) \right\|^2 + \frac{4 \sigma^2}{|\br|} \left( 1 + \frac{d}{\nR} \right) + \left( 1 + \frac{2}{|\br|} + \frac{2}{\nR |\br|} \right) \frac{\muR^2 L^2 d^2}{2} \right]. \label{eq_proof_thm_convex_1}
\end{align}
Here, note that
\begin{align}
    & \sum_{i=1}^d \frac{1}{p_{t,i}} \left( \nabla f (\mathbf{x}_{t}) \right)_i^2 \leq \frac{1}{\bar{c}} \left\| \nabla f (\mathbf{x}_{t}) \right\|^2 \label{eq_proof_thm_convex_2}
\end{align}
We take constant combination coefficients, i.e., $\alpha_t = \alpha$ for all $t$, and constant step-sizes $\eta_t = \eta$ for all $t$. 
We also use Assumption \ref{assum_bound_grad} to bound $\left\| \nabla f (\mathbf{x}_t) \right\|$. 
We denote $\bar{P}_T = \frac{1}{Td} \sum_{t=0}^{T-1} \sum_{i=1}^d \frac{1}{p_{t,i}}$. 
To focus only on the effect of the number of random directions in RGE $\nR$, and the probabilities of CGE $\{ p_{t,i} \}$ on convergence, we get rid of the sample set sizes using $|\brt| \geq 1, |\bct| \geq 1$ for all $t$. Dividing both sides of \eqref{eq_proof_thm_convex_1} by $T$, we get
\begin{align}
	\frac{1}{T} \sum_{t=1}^T \left( \mathbb{E} f(\mathbf{x}_t) - f^* \right) & \leq \frac{R^2}{2 \eta T} + R L \left( \frac{\alpha \muR d}{2} + (1-\alpha) \sqrt{d} \muC \right) \nonumber \\
	& \quad + \alpha^2 \eta \left[ 6 \left( 1 + \frac{d}{\nR} \right) G^2 + \left( \frac{3}{2} + \frac{1}{\nR} \right) \muR^2 d^2 L^2 + 4 \left( 1 + \frac{d}{\nR} \right) \sigma^2 \right] \nonumber \\
	& \quad + (1-\alpha)^2 \eta \left[ 2 G^2 \frac{1}{\bar{c}} + \bar{P}_T \left( 3 \sigma^2 + 2 L^2 \muC^2 d \right) \right]. \label{eq_proof_thm_convex_3}
\end{align}
As in Theorem \ref{thm_nonconvex}, we choose the smoothing parameters such that $\muC = \frac{\muR \sqrt{d}}{2}$. Again, using the same reasoning as in Theorem \ref{thm_nonconvex}, $\sigma^2 \geq L^2 \muC^2 d, \sigma^2 \geq \muR^2 L^2 d^2$. 
Also, note that $\bar{P}_T \leq \frac{1}{\bar{c}}$. 
Consequently \eqref{eq_proof_thm_convex_3} simplifies to
\begin{align}
	\frac{1}{T} \sum_{t=1}^T \left( \mathbb{E} f(\mathbf{x}_t) - f^* \right) & \leq \frac{R^2}{2 \eta T} + R L \sqrt{d} \muC + 6 \alpha^2 \eta \left[ G^2 \left( 1 + \frac{d}{\nR} \right) + \sigma^2 \left( 1 + \frac{d}{\nR} \right) \right] + 6 (1-\alpha)^2 \eta \left[ \frac{G^2 + \sigma^2}{\bar{c}} \right] \nonumber \\
	& \overset{(a)}{\leq} 2 \sqrt{\frac{3 R^2}{T} \left[ \alpha^2 \left( G^2 + \sigma^2 \right) \left( 1 + \frac{d}{\nR} \right) + (1-\alpha)^2 \left( \frac{G^2 + \sigma^2}{\bar{c}} \right) \right] } + R L \sqrt{d} \muC, \label{eq_proof_thm_convex_4}
\end{align}
where $(a)$ follows by optimizing over $\eta$. Choosing the value of $\alpha$
\begin{align}
	\alpha^* = \left[ 1 + \bar{c} \left( 1 + \frac{d}{\nR} \right) \right]^{-1} \label{proof_eq_thm_convex_alpha}
\end{align}
and substituting in \eqref{eq_proof_thm_convex_4}, we get
\begin{align}
    \frac{1}{T} \sum_{t=1}^T & \left( \mathbb{E} f(\mathbf{x}_t) - f^* \right) \leq O \left( R \sqrt{\frac{(G^2 + \sigma^2)}{T} \frac{\left( 1 + \frac{d}{\nR} \right)}{\left[ 1 + \bar{c} \left( 1 + \frac{d}{\nR} \right) \right]} } \right) + R L \sqrt{d} \muC. \label{eq_proof_thm_convex_5}
\end{align}
The choice of $\muC = O \left( \sqrt{\frac{1}{dT} \frac{\left( 1 + \frac{d}{\nR} \right)}{\left[ 1 + \bar{c} \left( 1 + \frac{d}{\nR} \right) \right]} } \right)$ in \eqref{eq_proof_thm_convex_5} yields
\begin{align}
    \frac{1}{T} \sum_{t=1}^T & \left( \mathbb{E} f(\mathbf{x}_t) - f^* \right) \leq O \left( R \sqrt{\frac{1}{T} \frac{\left( 1 + \frac{d}{\nR} \right)}{\left[ 1 + \bar{c} \left( 1 + \frac{d}{\nR} \right) \right]} } \right). \label{eq_proof_thm_convex_6}
\end{align}
\end{proof}

Similar to as we did in the nonconvex case (Appendix \ref{app_special_case_nonconvex}), we next explore the convergence conditions of Theorem \ref{thm_convex} under some special cases. The reasoning is quite similar in this case as well. 

\subsection{Special Cases of Theorem \ref{thm_convex} (Convex Case)}
\label{proof_convex_special_case}
Similar to as we did in the nonconvex case (Appendix \ref{app_special_case_nonconvex}), we next explore the convergence conditions of Theorem \ref{thm_convex} under some special cases. We summarize the results for three regimes in Table \ref{table_convex}, depending on the relative values of $1 + \tfrac{d}{\nR}$ and $\tfrac{1}{\bar{c}}$. The reasoning is quite similar.
\begin{table}[h!]
\begin{center}
\begin{tabular}{|c|c|c|}
\cline{1-3}
Regime & \begin{tabular}[c]{@{}c@{}} Smoothing \\
parameter $\muC$ \end{tabular} & \begin{tabular}[c]{@{}c@{}} Convergence \\ rate \end{tabular} \\
\cline{1-3}
$1 + \tfrac{d}{\nR} \gg \tfrac{1}{\bar{c}}$ & 
{\small$O \left( \sqrt{\tfrac{1}{d \bar{c} T}} \right)$}
& {\small$O \left(  \sqrt{ \tfrac{1}{T \bar{c}}} \right)$} \\
\cline{1-3}
$1 + \tfrac{d}{\nR} \ll \tfrac{1}{\bar{c}}$ & {\small$O \lp \sqrt{\tfrac{ 1 + d/\nR}{dT} } \rp$} & {\small$O \left( \sqrt{(1 + d/\nR)/T } \right)$} \\
\cline{1-3}
\begin{tabular}[c]{@{}c@{}} $1 + \tfrac{d}{\nR} \approx \tfrac{1}{\bar{c}}$ \\ $p_{t,i} \equiv \tfrac{\nC}{d}$ \end{tabular} & {\small$O \left( \sqrt{ \tfrac{1 + \tfrac{d}{\nR + \nC}}{(dT} } \right)$} & {\small$O \left( \sqrt{\tfrac{1}{T} \left( 1 + \tfrac{d}{\nR + \nC} \right)} \right)$} \\
\cline{1-3}
\end{tabular}
\caption{Comparison of smoothing parameters and convergence rates, for difference regimes of RGE and CGE query budgets, quantified by $\nR$, $\bar{c}$ respectively.}
\label{table_convex}
\end{center}
\end{table}
\subsubsection{Regime 1: $1 + \frac{d}{\nR} \gg \frac{1}{\bar{c}}$.}
Since $\frac{1}{\bar{c}} \geq 1$, this implies that $\frac{d}{\nR} \gg \frac{1}{\bar{c}}$, or $\bar{c} \gg \frac{\nR}{d}$. 
This implies that the sampling probabilities are much greater than $\nR/d$. 
In other words, the per-iteration query budget of CGE is much higher than RGE, and $\alpha \to 0$ (see \eqref{proof_eq_thm_convex_alpha}). 
With smoothing parameter $\muC = O \left( \sqrt{\frac{1}{d \bar{c} T}} \right)$, the resulting convergence rate is
\begin{align*}
	\mathbb{E} \left\| \nabla f (\bar{\mathbf{x}}_T) \right\|^2 \leq O \left(  \sqrt{ \frac{1}{T \bar{c}}} \right).
\end{align*}
To gain further insight, we consider the special case of uniform distribution, such that $p_{t,i} = \frac{\nC}{d}, \forall \ t,i$. Hence, $\frac{1}{\bar{c}} = \frac{d}{\nC}$, where $\frac{d}{\nR} \gg \frac{1}{\bar{c}}$ implies $\nC \gg \nR$. Consequently, the convergence rate becomes
{\small
\begin{align*}
	\mathbb{E} \left\| \nabla f (\bar{\mathbf{x}}_T) \right\|^2 \leq O \left( \sqrt{ \frac{d}{\nC T}} \right).
\end{align*}
}%
And the FQC to achieve $\mathbb{E} \| \nabla f (\bar{\mathbf{x}}_T) \|^2 \leq \epsilon$ is $O(T \cdot \nC + T \cdot \nR) = O(T \cdot \nC) = O(d/\epsilon^2)$.

\subsubsection{Regime 2: $1 + \frac{d}{\nR} \ll \frac{1}{\bar{c}}$.}
Since $\frac{1}{\bar{c}} \geq 1$, this implies that $\frac{d}{\nR} \ll \frac{1}{\bar{c}}$, or $\bar{c} \ll \frac{\nR}{d}$. 
This implies that 
A sufficient condition under which this holds is if the sampling probabilities are all much smaller than $\nR/d$. In other words, the per-iteration query budget of CGE is much smaller than RGE, and $\alpha \to 1$ (see \eqref{proof_eq_thm_convex_alpha}). 
With smoothing parameter $\muC = O \left( \sqrt{\frac{ 1 + \frac{d}{\nR}}{dT} } \right)$, the resulting convergence rate is
\begin{align*}
	\mathbb{E} \left\| \nabla f (\bar{\mathbf{x}}_T) \right\|^2 \leq O \left( \sqrt{\frac{1 + \frac{d}{\nR}}{T} } \right).
\end{align*}
To gain further insight, we consider the special case of uniform distribution, such that $p_{t,i} = \frac{\nC}{d}, \forall \ t,i$. Hence, $\frac{1}{\bar{c}} = \frac{d}{\nC}$, where $\frac{d}{\nR} \ll \frac{1}{\bar{c}}$ implies $\nC \ll \nR$. Consequently, the FQC to achieve $\mathbb{E} \| \nabla f (\bar{\mathbf{x}}_T) \|^2 \leq \epsilon$ is $O(T \cdot \nC + T \cdot \nR) = O(T \cdot \nR) = O(d/\epsilon^2)$ (assuming $\nR = O(d)$).

\subsubsection{Regime 3: $\dnr = 1 + \frac{d}{\nR}$ and $\frac{1}{\bar{c}}$ are comparable in value.}
To gain some insight into this case where the function query budgets of RGE and CGE are comparable, we again look at the uniform distribution $p_{t,i} = \frac{\nC}{d}, \forall \ t,i$. So, $\frac{1}{\bar{c}} = \frac{d}{\nC}$.
The total per-iteration function query cost of HGE is $O(\nR + \nC)$.
First, note that
\begin{align}
	\frac{1 + \frac{d}{\nR}}{1 + \bar{c} \left( 1 + \frac{d}{\nR} \right)} &= \frac{1 + \frac{d}{\nR}}{1 + \left( 1 + \frac{d}{\nR} \right) \frac{\nC}{d}} = \frac{d \nR + d^2}{\nR \nC + d(\nR + \nC)} \nonumber \\
	& \leq \frac{\nR}{\nR + \nC} + \frac{d}{\nR + \nC} \leq 1 + \frac{d}{\nR + \nC}. \nonumber
\end{align}
With $\muC = O \left( \sqrt{\frac{1}{dT} \left( 1 + \frac{d}{\nR + \nC} \right) } \right)$ and $\muR = O \left( \frac{1}{d} \sqrt{\frac{1}{T} \left( 1 + \frac{d}{\nR + \nC} \right)} \right)$, the resulting convergence rate is
\begin{align*}
	\mathbb{E} \left\| \nabla f (\bar{\mathbf{x}}_T) \right\|^2 \leq O \left( \sqrt{\frac{1}{T} \left( 1 + \frac{d}{\nR + \nC} \right)} \right).
\end{align*}
FQC to achieve $\mathbb{E} \| \nabla f (\bar{\mathbf{x}}_T) \|^2 \leq \epsilon$ is given by $ O(T \cdot\nR+ T \cdot \nC) = O(d/\epsilon^2)$ (assuming $\nR + \nC = O(d)$). 
For $O(\nR + \nC)$ function evaluations per iteration ($O(\nR)$ for RGE, $O(\nC)$ for CGE), this rate is order optimal \cite{duchi15optimal_tit}. 

\subsection{Proof of Theorem \ref{thm_convex} (Strongly Convex Case)}
For strongly convex functions, the error can be expressed either in terms of the \textit{regret} $\sum_{t=0}^{T-1} \left( \mathbb{E} f (\mathbf{x}_t) - f^* \right)$ (as in the convex case), or in terms of distance of the iterate from the optima $\| \mathbf{x}_t - \mathbf{x}^* \|^2$, because of the following inequality
\begin{align*}
	\frac{\bar{\sigma}}{2} \left\| \mathbf{x} - \mathbf{x}^* \right\|^2 \leq f(\mathbf{x}) - f^*, \qquad \forall \ \mathbf{x} \in \mathrm{ dom } f.
\end{align*}
We begin with the following general result.
\begin{lemma}
\label{lem_strong_convex_smooth}
Suppose $f$ satisfies Assumption \ref{assum_lipschitz}, \ref{assum_strong_convex}. Then, the smooth approximation $f_{\mu}$ of the function $f$, defined as $f_{\mu} = \mathbb{E}_{\mbf u \in U_0} [ f(\mathbf{x} + \mu \mathbf{u}) ]$ is also $\bar{\sigma}$-strongly convex.
\end{lemma}

\begin{proof}
We use the following property of strongly convex functions.
\begin{align*}
    f(\alpha \mathbf{x} + (1-\alpha) \mbf y) \leq \alpha f(\mbf x) + (1-\alpha) f(\mbf y) - \frac{\bar{\sigma} \alpha (1-\alpha)}{2} \left\| \mbf x - \mbf y \right\|^2, \qquad \alpha \in [0,1].
\end{align*}
Now, 
\begin{align}
    f_{\mu} (\alpha \mathbf{x} + (1-\alpha) \mbf y) &= \mathbb{E}_{\mbf u \in U_0} \left[ f \left( \alpha \mathbf{x} + (1-\alpha) \mbf y + \mu \mbf u \right) \right] \nonumber \\
    &= \mathbb{E}_{\mbf u \in U_0} \left[ f \left( \alpha ( \mathbf{x} + \mu \mbf u) + (1-\alpha) (\mbf y + \mu \mbf u) \right) \right] \nonumber \\
    & \leq \mathbb{E}_{\mbf u \in U_0} \left[ \alpha f(\mbf x + \mu \mbf u) + (1-\alpha) f(\mbf y + \mu \mbf u) - \frac{\bar{\sigma} \alpha (1-\alpha)}{2} \left\| \mbf x + \mu \mbf u - \mbf y - \mu \mbf u \right\|^2 \right] \nonumber \\
    &= \alpha \mathbb{E}_{\mbf u \in U_0} \left[ f(\mbf x + \mu \mbf u) \right] + (1-\alpha) \mathbb{E}_{\mbf u \in U_0} \left[ f(\mbf y + \mu \mbf u) \right] - \frac{\bar{\sigma} \alpha (1-\alpha)}{2} \left\| \mbf x - \mbf y \right\|^2 \nonumber \\
    &= \alpha f_{\mu}(\mbf x) + (1-\alpha) f_{\mu}(\mbf y) - \frac{\bar{\sigma} \alpha (1-\alpha)}{2} \left\| \mbf x - \mbf y \right\|^2 \nonumber
\end{align}
\end{proof}

We start with the following intermediate result.
\begin{prop}
\label{prop_strong_convex_1}
Given $f$ satisfies Assumption \ref{assum_lipschitz}, \ref{assum_strong_convex}, the iterates of Algorithm \ref{Algo_zo_hgd} $\{\mbf x_t \}_t$ satisfy
\begin{align}
	\mathbb{E} \left\| \mathbf{x}_{t+1} - \mathbf{x}^* \right\|^2 &\leq \left( 1 - \frac{\eta_t \bar{\sigma}}{2} \right) \mathbb{E} \left\| \mathbf{x}_{t} - \mathbf{x}^* \right\|^2 - 2 \eta_t e_t + 2 \eta_t \alpha_t L \muR^2 + 2 \eta_t (1-\alpha_t) \frac{L^2 d \muC^2}{\bar{\sigma}} \nonumber \\
	& \quad + 2 \eta_t^2 \left[ \alpha_t^2 \mathbb{E} \left\| \nabla_{\mathrm{r},t} \right\|^2 + (1-\alpha_t)^2 \mathbb{E} \left\| \nabla_{\mathrm{c},t} \right\|^2 \right],
\end{align}
where, $e_t = \mathbb{E} f (\mathbf{x}_t) - f^*$. Recall that $\eta_t$ is the step-size at time $t$, $\alpha_t$ is the combination coefficient at time $t$, $\muR, \muC$ are the smoothing parameters for RGE, CGE respectively.
\end{prop}

\begin{proof}
Using the iterate update equation (line \ref{line_update}) in Algorithm \ref{Algo_zo_hgd}
\begin{align}
	\left\| \mathbf{x}_{t+1} - \mathbf{x}^* \right\|^2 &= \left\| \mathbf{x}_{t} - \eta_t \left( \alpha_t \nabla_{\mathrm{r},t} + (1-\alpha_t) \nabla_{\mathrm{c},t} \right) - \mathbf{x}^* \right\|^2 \nonumber \\
	&= \left\| \mathbf{x}_{t} - \mathbf{x}^* \right\|^2 + \eta_t^2 \left\| \alpha_t \nabla_{\mathrm{r},t} + (1-\alpha_t) \nabla_{\mathrm{c},t} \right\|^2 - 2 \eta_t \left\langle \mathbf{x}_{t} - \mathbf{x}^*, \alpha_t \nabla_{\mathrm{r},t} + (1-\alpha_t) \nabla_{\mathrm{c},t} \right\rangle. \label{eq_dist_to_opt_1}
\end{align}
We take expectation and consider the individual terms in \eqref{eq_dist_to_opt_1} one at a time.
\begin{align}
	\mathbb{E} \left\| \alpha_t \nabla_{\mathrm{r},t} + (1-\alpha_t) \nabla_{\mathrm{c},t} \right\|^2 \leq 2 \alpha_t^2 \mathbb{E} \left\| \nabla_{\mathrm{r},t} \right\|^2 + 2(1-\alpha_t)^2 \mathbb{E} \left\| \nabla_{\mathrm{c},t} \right\|^2, \label{eq_dist_to_opt_term1}
\end{align}
where \eqref{eq_dist_to_opt_term1} follows from $\|a+b\|^2 \leq 2\| a \|^2 + 2\| b \|^2$. Next,
\begin{align}
	& -\eta_t \mathbb{E} \left[ \left\langle \mathbf{x}_{t} - \mathbf{x}^*, \alpha_t \nabla_{\mathrm{r},t} + (1-\alpha_t) \nabla_{\mathrm{c},t} \right\rangle \mid \mathcal{Y}_t \right] \nonumber \\
	& \qquad = -\eta_t \alpha_t \mathbb{E} \left\langle \mathbf{x}_{t} - \mathbf{x}^*, \nabla f_{\muR} (\mathbf{x}_t) \right\rangle -\eta_t (1-\alpha_t) \mathbb{E} \left\langle \mathbf{x}_{t} - \mathbf{x}^*, \Hat{\nabla}_{\mathrm{CGE}} f(\mathbf{x}_t) \right\rangle.
\end{align}
Using Lemma \ref{lem_strong_convex_smooth}, we can upper bound
\begin{align}
    -\left\langle \mathbf{x}_{t} - \mathbf{x}^*, \nabla f_{\muR} (\mathbf{x}_t) \right\rangle &\leq - \left( f_{\muR} (\mathbf{x}_t) - f_{\muR} (\mathbf{x}^*) \right) - \frac{\bar{\sigma}}{2} \left\| \mathbf{x}_{t} - \mathbf{x}^* \right\|^2 \nonumber \\
    & \leq - \left( f (\mathbf{x}_t) - f^* \right) + \muR^2 L - \frac{\bar{\sigma}}{2} \left\| \mathbf{x}_{t} - \mathbf{x}^* \right\|^2, \label{eq_dist_to_opt_term_innerprod1}
\end{align}
where, \eqref{eq_dist_to_opt_term_innerprod1} follows from \cite[Lemma~5.1]{zo_spider_liang19icml}.
Also,
\begin{align}
    -\left\langle \mathbf{x}_{t} - \mathbf{x}^*, \Hat{\nabla}_{\mathrm{CGE}} f(\mathbf{x}_t) \right\rangle &= -\left\langle \mathbf{x}_{t} - \mathbf{x}^*, \Hat{\nabla}_{\mathrm{CGE}} f(\mathbf{x}_t) - \nabla f(\mathbf{x}_t) \right\rangle - \left\langle \mathbf{x}_{t} - \mathbf{x}^*, \nabla f(\mathbf{x}_t) \right\rangle \nonumber \\
    & \overset{(a)}{\leq} \left\| \mathbf{x}_{t} - \mathbf{x}^* \right\| \left\| \Hat{\nabla}_{\mathrm{CGE}} f(\mathbf{x}_t) - \nabla f(\mathbf{x}_t) \right\| - \left( f (\mathbf{x}_t) - f (\mathbf{x}^*) \right) - \frac{\bar{\sigma}}{2} \left\| \mathbf{x}_{t} - \mathbf{x}^* \right\|^2 \nonumber \\
    & \overset{(b)}{\leq} \frac{\bar{\sigma}}{4} \left\| \mathbf{x}_{t} - \mathbf{x}^* \right\|^2 + \frac{1}{\bar{\sigma}} \left\| \Hat{\nabla}_{\mathrm{CGE}} f(\mathbf{x}_t) - \nabla f(\mathbf{x}_t) \right\|^2 - \left( f (\mathbf{x}_t) - f (\mathbf{x}^*) \right) - \frac{\bar{\sigma}}{2} \left\| \mathbf{x}_{t} - \mathbf{x}^* \right\|^2 \nonumber \\
    & \overset{(c)}{\leq} \frac{1}{\bar{\sigma}} L^2 d \muC^2 - \left( f (\mathbf{x}_t) - f^* \right) - \frac{\bar{\sigma}}{4} \left\| \mathbf{x}_{t} - \mathbf{x}^* \right\|^2, \label{eq_dist_to_opt_term_innerprod2}
\end{align}
where $(a)$ follows from Cauchy-Schwarz inequality, and strong convexity of $f$; $(b)$ follows from Young's inequality $\mathbf{x}^T \mathbf{y} \leq \frac{a}{2} \| \mathbf{x} \|^2 + \frac{1}{2a} \| \mathbf{y} \|^2$; $(c)$ follows from \cite[Lemma~3]{zo_spider_liang19icml}
Substituting \eqref{eq_dist_to_opt_term1}, \eqref{eq_dist_to_opt_term_innerprod1}, \eqref{eq_dist_to_opt_term_innerprod2} in \eqref{eq_dist_to_opt_1}, we get the statement of the Lemma.
\end{proof}
Next, we can upper bound $\mathbb{E} \left\| \nabla_{\mathrm{r},t} \right\|^2$, $\mathbb{E} \left\| \nabla_{\mathrm{c},t} \right\|^2$ using Proposition \ref{prop_bd_norm_sq_rge}, and $p_{t,i} \geq \bar{c} > 0$  for all $t,i$. 

\begin{prop}
\label{prop_strong_convex_bd_norm}
Suppose $f$ satisfies Assumption \ref{assum_lipschitz}, \ref{assum_var_bound}, \ref{assum_bound_grad}. Then
\begin{align*}
    \mathbb{E} \left\| \nabla_{\mathrm{r},t} \right\|^2 & \leq  6 (G^2 + \sigma^2) \left( 1 + \frac{d}{\nR} \right) \\
    \mathbb{E} \left\| \nabla_{\mathrm{c},t} \right\|^2 & \leq 6 \frac{(G^2 + \sigma^2)}{\bar{c}}.
\end{align*}
\end{prop}

\begin{proof}
\begin{align}
    \mathbb{E} \left\| \nabla_{\mathrm{r},t} \right\|^2 & \leq \left[ 2 + \frac{4}{|\br|} \left( 1 + \frac{d}{\nR} \right) \right] \mathbb{E} \left\| \nabla f (\mathbf{x}) \right\|^2 + \frac{4 \sigma^2}{|\br|} \left( 1 + \frac{d}{\nR} \right) + \left( 1 + \frac{2}{|\br|} + \frac{2}{\nR |\br|} \right) \frac{\muR^2 L^2 d^2}{2} \nonumber \\
    & \overset{(a)}{\leq} 6 \left( 1 + \frac{d}{\nR} \right) \mathbb{E} \left\| \nabla f (\mathbf{x}) \right\|^2 + 4 \sigma^2 \left( 1 + \frac{d}{\nR} \right) + \left( \frac{3}{2} + \frac{1}{\nR} \right) \muR^2 L^2 d^2 \nonumber \\
    & \overset{(b)}{\leq} 6 \left( 1 + \frac{d}{\nR} \right) G^2 + 6 \sigma^2 \left( 1 + \frac{d}{\nR} \right) \nonumber \\
    &= 6 (G^2 + \sigma^2) \left( 1 + \frac{d}{\nR} \right), \label{eq_proof_str_convex_bd_rge}
\end{align}
where $(a)$ follows since $|\brt| \geq 1$; $(b)$ follows from Assumption \ref{assum_bound_grad}, and by using the smoothing parameter $\muR$ small enough such that $\sigma \geq \muR L d$. Next,
\begin{align}
    \mathbb{E} \left\| \nabla_{\mathrm{c},t} \right\|^2 & \leq \sum_{i=1}^d \frac{1}{p_{t,i}} \Big[ 2 \mathbb{E} \left( \nabla f (\mathbf{x}_{t}) \right)_i^2 + \frac{3 \zeta^2}{|\bct|} + \frac{L^2 \muC^2}{2} \left( 1 + \frac{3}{|\bct|} \right) \Big] \nonumber \\
    & \overset{(c)}{\leq} \frac{2 \mathbb{E} \left\| \nabla f (\mathbf{x}) \right\|^2}{\bar{c}} + \frac{3 \sigma^2}{\bar{c}} + \frac{2 L^2 \muC^2 d}{\bar{c}} \nonumber \\
    & \overset{(d)}{\leq} 6 \frac{(G^2 + \sigma^2)}{\bar{c}}, \label{eq_proof_str_convex_bd_cge}
\end{align}
where $(c)$ follows since $p_{t,i} \geq \bar{c}$, $|\bct| \geq 1$, and $\sigma^2 = d \zeta^2$; $(d)$ follows from Assumption \ref{assum_bound_grad}, and by choosing $\muC$ small enough such that $\sigma \geq L \muC \sqrt{d}$.
\end{proof}
Substituting the bounds from Proposition \ref{prop_strong_convex_bd_norm} in Proposition \ref{prop_strong_convex_1}, and assuming constant combination coefficient $\alpha_t = \alpha$, for all $t$, we get
\begin{align}
	\mathbb{E} \left\| \mathbf{x}_{t+1} - \mathbf{x}^* \right\|^2 &\leq \left( 1 - \frac{\eta_t \bar{\sigma}}{2} \right) \mathbb{E} \left\| \mathbf{x}_{t} - \mathbf{x}^* \right\|^2 - 2 \eta_t e_t + 2 \eta_t \alpha L \muR^2 + 2 \eta_t (1-\alpha) \frac{L^2 d \muC^2}{\bar{\sigma}} \nonumber \\
	& \quad + 12 \eta_t^2 (G^2 + \sigma^2) \left[ \alpha^2 \left( 1 + \frac{d}{\nR} \right) + (1-\alpha)^2 \frac{1}{\bar{c}} \right], \nonumber \\
	& \overset{(e)}{\leq} \left( 1 - \frac{\eta_t \bar{\sigma}}{2} \right) \mathbb{E} \left\| \mathbf{x}_{t} - \mathbf{x}^* \right\|^2 - 2 \eta_t e_t + 2 \eta_t \frac{L^2 d \muC^2}{\bar{\sigma}} + 12 \eta_t^2 (G^2 + \sigma^2) \frac{\left( 1 + \frac{d}{\nR} \right)}{1 + \bar{c} \left( 1 + \frac{d}{\nR} \right)}, \label{eq_thm_strong_convex_1}
\end{align}
where $(e)$ follows by choosing $\muR = \muC \sqrt{\frac{d L}{\bar{\sigma}}}$, and by choosing $\alpha$ to minimize the right hand side. Next, we state the following general result to help us bound \eqref{eq_thm_strong_convex_1}.

\begin{lemma}
\label{lem_stich_paper}
Let $\{ a_t \}_{t \geq 0}, a_t \geq 0$, $\{ e_t \}_{t \geq 0}, e_t \geq 0$ be sequences satisfying 
\begin{align*}
	a_{t+1} \leq \left( 1 - \frac{\eta_t \bar{\sigma}}{2} \right) a_t - \eta_t e_t + \eta_t A + \eta_t^2 B,
\end{align*}
for $\eta_t = \frac{8}{\bar{\sigma} (a+t)}$, and $A, B \geq 0$, $\bar{\sigma} > 0, a > 1$. Then,
\begin{align}
	\frac{1}{S_T} \sum_{t=0}^{T-1} w_t e_t \leq A + \frac{4 T (T + 2a)}{\bar{\sigma} S_T} B + \frac{a^3 \bar{\sigma}}{8 S_T} a_0,
\end{align}
for $w_t = (a+t)^2$ and $S_T = \sum_{t=0}^{T-1} w_t \geq \frac{1}{3} T^3$.
\end{lemma}

\begin{proof}
The proof borrows from the proof in \cite[Lemma~3.3]{stich18spars_SGD} with some minor modifications. Multiplying  by $\frac{w_t}{\eta_t}$, and simplifying, we get
\begin{align}
	\sum_{t=0}^{T-1} w_t e_t \leq \frac{w_0}{\eta_0} a_0 + A \sum_{t=0}^{T-1} w_t + \sum_{t=0}^{T-1} w_t \eta_t B.
\end{align}
Here, $\frac{w_0}{\eta_0} \leq \frac{\bar{\sigma} a^3}{8}$. Using $S_T = \sum_{t=0}^{T-1} w_t \geq T^3/3$, $\sum_{t=0}^{T-1} w_t \eta_t \leq \frac{4 T (T + 2a)}{\bar{\sigma}}$, we get the result.
\end{proof}

Comparing \eqref{eq_thm_strong_convex_1} and Lemma \ref{lem_stich_paper}, note that
\begin{align*}
    e_t = \frac{1}{2} \mathbb{E} \left\| \mathbf{x}_{t} - \mathbf{x}^* \right\|^2, \qquad A = \frac{L^2 d \muC^2}{\bar{\sigma}}, \qquad B = 6 (G^2 + \sigma^2) \frac{\left( 1 + \frac{d}{\nR} \right)}{1 + \bar{c} \left( 1 + \frac{d}{\nR} \right)}.
\end{align*}
Then, for $\Hat{\mathbf{x}}_T = \frac{1}{S_T} \sum_{t=0}^{T-1} w_t \mathbf{x}_t$,
\begin{align}
    \mathbb{E} f(\Hat{\mathbf{x}}_T) - f^* & \leq \frac{1}{S_T} \sum_{t=0}^{T-1} w_t e_t \leq A + \frac{4 T (T + 2a)}{\bar{\sigma} S_T} B + \frac{a^3 \bar{\sigma}}{8 S_T} a_0 \nonumber \\
    & \leq O \left( \frac{L^2 d \muC^2}{\bar{\sigma}} + \frac{(G^2 + \sigma^2)}{\bar{\sigma} T} \frac{\left( 1 + \frac{d}{\nR} \right)}{1 + \bar{c} \left( 1 + \frac{d}{\nR} \right)} \right). \label{eq_thm_strong_convex_2}
\end{align}
The choice of $\muC = O \left( \sqrt{\frac{1}{dT} \frac{\left( 1 + \frac{d}{\nR} \right)}{\left[ 1 + \bar{c} \left( 1 + \frac{d}{\nR} \right) \right]} } \right)$ in \eqref{eq_thm_strong_convex_2} yields
\begin{align}
    \mathbb{E} f(\Hat{\mathbf{x}}_T) - f^* & \leq O \left( \frac{(G^2 + \sigma^2)}{\bar{\sigma} T} \frac{\left( 1 + \frac{d}{\nR} \right)}{1 + \bar{c} \left( 1 + \frac{d}{\nR} \right)} \right). \label{eq_thm_strong_convex_3}
\end{align}
The special cases of Theorem \ref{thm_convex} (for strongly convex functions) can be derived in a similar way, as we did for convex functions in Appendix \ref{proof_convex_special_case}.

\end{appendices}

\end{document}